\documentclass[sigconf]{acmart}

\AtBeginDocument{%
  \providecommand\BibTeX{{%
    \normalfont B\kern-0.5em{\scshape i\kern-0.25em b}\kern-0.8em\TeX}}}

\setcopyright{acmcopyright}
\copyrightyear{2022}
\acmYear{2022}
\acmDOI{10.1145/1122445.1122456}

\acmConference[GECCO-2022]{GECCO-2022: Genetic and Evolutionary Computation Conference}{July 7--13, 2022}{Boston, MA}
\acmBooktitle{GECCO-2022:  Genetic and Evolutionary Computation Conference, July 7--13, 2022, Boston, MA}
\acmPrice{15.00}
\acmISBN{978-1-4503-XXXX-X/18/06}

\usepackage{algorithm}
\usepackage{algorithmic}

\usepackage{amsfonts}

\usepackage{multirow}
\usepackage{siunitx}
\usepackage{amsmath}

\usepackage{amssymb}

\usepackage{amsthm}

\usepackage{siunitx}
\usepackage{booktabs}

\newcommand{\akparen}[1]{\left( #1 \right)}
\newcommand{\akbrack}[1]{\left[ #1 \right]}
\newcommand{\akset}[1]{\left\{ #1 \right\}}

\DeclareMathOperator*{\argmin}{arg\,min}
\DeclareMathOperator{\EX}{\mathbb{E}}
\usepackage{mathtools}
\DeclarePairedDelimiter\floor{\lfloor}{\rfloor}

\newtheorem{theorem}{Theorem}

\begin{document}

\title{
Effective Mutation Rate Adaptation\\through Group Elite Selection
}


\author{Akarsh Kumar}
\affiliation{%
  \institution{The University of Texas at Austin}
  \streetaddress{1 Th{\o}rv{\"a}ld Circle}
  \city{Austin} \state{Texas} \country{USA}
  }
\email{akarshkumar0101@gmail.com}

\author{Bo Liu}
\affiliation{%
  \institution{The University of Texas at Austin}
  \city{Austin} \state{Texas} \country{USA}
}
\email{bliu@cs.utexas.edu}

\author{Risto Miikkulainen}
\affiliation{%
  \institution{The University of Texas at Austin}
  \institution{and Cognizant AI Labs}
  \streetaddress{30 Shuangqing Rd}
  \city{Austin} \state{Texas} \country{USA}
  }
\email{risto@cs.utexas.edu}

\author{Peter Stone}
\affiliation{%
  \institution{The University of Texas at Austin}
  \institution{and Sony AI}
  \streetaddress{8600 Datapoint Drive}
  \city{Austin} \state{Texas} \country{USA}
  }
\email{pstone@cs.utexas.edu}


\renewcommand{\shortauthors}{Kumar, et al.}

\begin{abstract}
Evolutionary algorithms are sensitive to the mutation rate (MR);
no single value of this parameter works well across domains.
Self-adaptive MR approaches have been proposed 
but they tend to be brittle: Sometimes they decay the MR to zero, thus halting evolution.
To make self-adaptive MR robust, this paper introduces the Group Elite Selection of Mutation Rates (GESMR) algorithm. 
\linebreak
GESMR co-evolves a population of solutions and a population of MRs, such that each MR is assigned to a group of solutions.
The resulting \emph{best} mutational change in the group, instead of average mutational change, is used for MR selection during evolution, thus avoiding the vanishing MR problem.
With the same number of function evaluations
and with almost no overhead, GESMR converges faster
and to better solutions than previous approaches on a wide range of continuous test optimization problems. GESMR also scales well to high-dimensional
neuroevolution for supervised image-classification tasks and for reinforcement learning control tasks. 
Remarkably,
GESMR produces MRs that are \emph{optimal in the long-term}, as demonstrated through a comprehensive look-ahead grid search.
Thus, GESMR and its theoretical and empirical analysis demonstrate how
self-adaptation can be harnessed to improve performance in several
applications of evolutionary computation.

\end{abstract}




\maketitle

\section{Introduction}
Biological evolution has produced an incredible diversity of life that is seen everywhere.
In this process, the solutions and the
\linebreak
mechanisms co-evolve end-to-end,
including the mutation rate
\linebreak
\cite[MR;][]{Metzgar2000-ji}.
Self-adaptation of MRs (SAMR)
is a technique common in the literature
of genetic algorithms (GA) that
encapsulates this idea of end-to-end
evolution of the MR along
with the individuals
\cite{Meyer-Nieberg2007-nz,Back1992-al,Smith1996-cg,Dang2016-fp}.
The idea is to assign each individual its
own MR, creating a pair.
The pairs are then evolved
end-to-end using the assigned MR
for mutating the individual and a ``meta" MR
for mutating the assigned MR.

However, this approach often
runs into the problem that
the MRs produced decay to zero,
causing evolution to stop at a sub-optimal value.
If instead the MR were fixed at some
moderate value, evolution would continue and find a better function value
\cite{Clune2008-lh, Rudolph2001-zf,Glickman2000-mr}.
This premature convergence can be attributed to
the fact that most mutations
hurt the fitness of an individual \cite{Clune2008-lh},
and thus an effective way for an individual to preserve its fitness into the next
generation is to have no mutation.
Thus, SAMR ignores the long-term goal of evolution to explore the fitness landscape and find
better solutions in future generations \cite{Clune2008-lh}.

To counteract this effect, 
this paper proposes a novel GA based on
supportive co-evolution \cite{Goldman2012-hc}
of solutions and MRs, entitled
\linebreak
Group Elite Selection of Mutation Rates (GESMR).
After assigning each MR to a group of solutions,
the solutions are evolved using that
MR, and the MRs are evolved according to the
\emph{best} change in function value from the MR's solution group,
defined as the ``group elite".
By targeting the MR that produces the \emph{best}
change in function value, given many mutation samples,
GESMR can mitigate the vanishing MR problem.
Additionally, GESMR is straightforward to implement and 
requires no more function evaluations than a fixed MR GA, 
and thus can be applied to a wide range of GA problems.

In prior work, a related approach using the idea of group elites was formulated as a multi-armed bandit problem and applied to entire genetic operators in an ad-hoc manner
\cite{Fialho2008-ud, Whitacre2009-uv}.
In contrast, this paper demonstrates that the approach is most effective when focused on MRs, and it also makes it possible to understand this result both empirically and theoretically.

Evaluation of GESMR is performed on common
benchmark test optimization problems from the GA literature.
To show that the method scales well to harder problems, it is also evaluated on
neuroevolution for image classification in the MNIST/Fashion-MNIST domain
and on reinforcement learning for control in the
CartPole, Pendulum, Acrobot, and MountainCar domains.
For comparison, results of
several adaptive MR algorithms including 
an oracle optimal fixed MR,
an oracle look-ahead MR (that uses foresight to determine MR),
self-adaptive MR,
the multi-armed bandit method \cite{Fialho2008-ud},
and some common heuristic methods \cite{Rechenberg1978-np}
are also reported.

GESMR outperforms other algorithms in most tasks.
Even when SAMR prematurely converges,
like in problems with especially 
\linebreak
rugged fitness landscapes
\cite{Clune2008-lh}, GESMR does not.
As a matter of fact, GESMR performs as well as the oracle look-ahead MR 
in function value and even matches the MR to the empirically estimated \emph{long-term optimal MR}.
To explain why, the statistical distribution
of the change in function value for a spectrum of MRs
for different function landscapes is empirically analyzed and visualized.
This analysis shows that SAMR is minimizing an MR objective
whose optimal MR is zero in rugged landscapes, while 
GESMR is minimizing an objective whose optimal MR is nonzero.


\section{Related Work}

Research on mutation rates (MRs)
is one of the most studied
sub-fields of genetic algorithms
\cite{Aleti2016-wn, Eiben1999-ez, Karafotias2015-cr, Kramer2010-ia, Hassanat2019-iz, Back1996-tv}. 
\paragraph{Fixed MRs:}
Lots of theoretical and empirical work has been done
on finding the optimal fixed MR for specific problems 
\cite{Greenwell1995-kx, Bottcher2010-xw},
finding heuristics like the MR should be proportional to  $1/L$ where 
$L$ is the length of the genotype \cite{Ochoa2002-ni, Doerr2019-ol}.
Evolutionary bilevel optimization tries to find the optimal evolutionary parameters, including MR,
by running an inner evolution with an outer loop searching over parameters \cite{Sinha2018-rd, Liang2015-dg}.
However, it is commonly known that the optimal MR
is constantly changing during evolution \cite{Patnaik1986-qb}.

\paragraph{Deterministic MRs:}
Deterministic MRs are common but these are ad hoc 
functions to change
the MR as a function of the number of generations,
and may not generalize to unseen problems with different landscapes
\cite{Aleti2016-wn}.

\paragraph{Adaptive MRs:}
Adaptive MRs are also common 
\cite{Thierens2002-dq, Srinivas1994-by, Patnaik1986-qb, Doerr2019-ol, Sewell2006-gp}
but these rely on another ad hoc system to determine
how to alter the MR given feedback from the evolution.
A common technique is to maintain a MR that
produces mutations of which only one-fifth are beneficial 
\cite{Karafotias2015-cr, Rechenberg1978-np},
by increasing MR when the percentage of successful mutations is 
greater than $1/5$ (and vice versa).
Although this technique is based on empirical findings,
it is ad-hoc, does not generalize to different landscapes,
requires a hardcoded threshold, and has been shown to lead to premature convergence 
when elitism is employed \cite{Rudolph2001-zf}.

\paragraph{Self-Adaptive MRs:}
Perhaps the most promising and 
evolutionarily plausible class
of adapting MRs is that of 
self-adapting MRs 
\cite{Kramer2010-ia, Aleti2016-wn, Back1992-al, Gomez2004-sh, Thierens2002-dq}.
This technique concatenates an MR to each individual
and evolves the MRs and individuals in one end-to-end evolutionary process.
However, many previous works have shown this process
to be brittle and lead to premature convergence of evolution
as the MRs decay and vanish 
\cite{Rudolph2001-zf, Glickman2000-mr, Clune2008-lh, Meyer-Nieberg2007-nz}.
In the instances where self-adapting MRs succeed,
the authors attribute the cause to be from a relatively 
smooth fitness landscape
\cite{Clune2008-lh, Glickman2000-mr}, 
or high selection pressure
\cite{Maschek2010-ap}.
The cause of general premature convergence in rugged landscapes is attributed to
the fact that most mutations are deleterious,
causing self-adaptation to prefer solutions that mutate
less and preserve the fitness of each individual
\cite{Clune2008-lh, Glickman2000-mr}.
\citet{Clune2008-lh} mention that, in this way, 
evolution is short-sighted: it cannot adapt MRs 
to be optimal for the long-term,
only optimizing for short-term performance.

\paragraph{Outlier-Based MRs:}
Some works have proposed looking
at the best mutation produced by a certain mutation operator
to judge the quality of the operator \cite{Fialho2008-ud, Whitacre2009-uv}, 
with the motivation that an operator that produces 
infrequent large fitness gains is
preferred to one that produces frequent small fitness gains.
However, these works model the operator selection
as a multi-armed bandit problem.
This technique is not only unnatural to evolution,
it is also limited by the expressiveness of the arms used
and assumes independent arms, thus failing to capture 
the continuous spectrum that the MR exists in.

\paragraph{CMA-ES:}
One of the most successful forms of adapting the
\linebreak
spread of a population during an evolutionary search is with
Covariance Matrix Adaptation Evolution Strategy (CMA-ES) \cite{Hansen2016-qa}.
It relies on maintaining a covariance matrix, which requires quadratic time and space in the solution vector length. Thus, CMA-ES does not scale to larger problems like deep neuroevolution with millions of parameters \cite{Such2017-rn}. In contrast, GESMR and GAs in general are linear wrt.\ solution length.

\section{Method}

\begin{figure*}
    \centering
    \includegraphics[width=\textwidth]{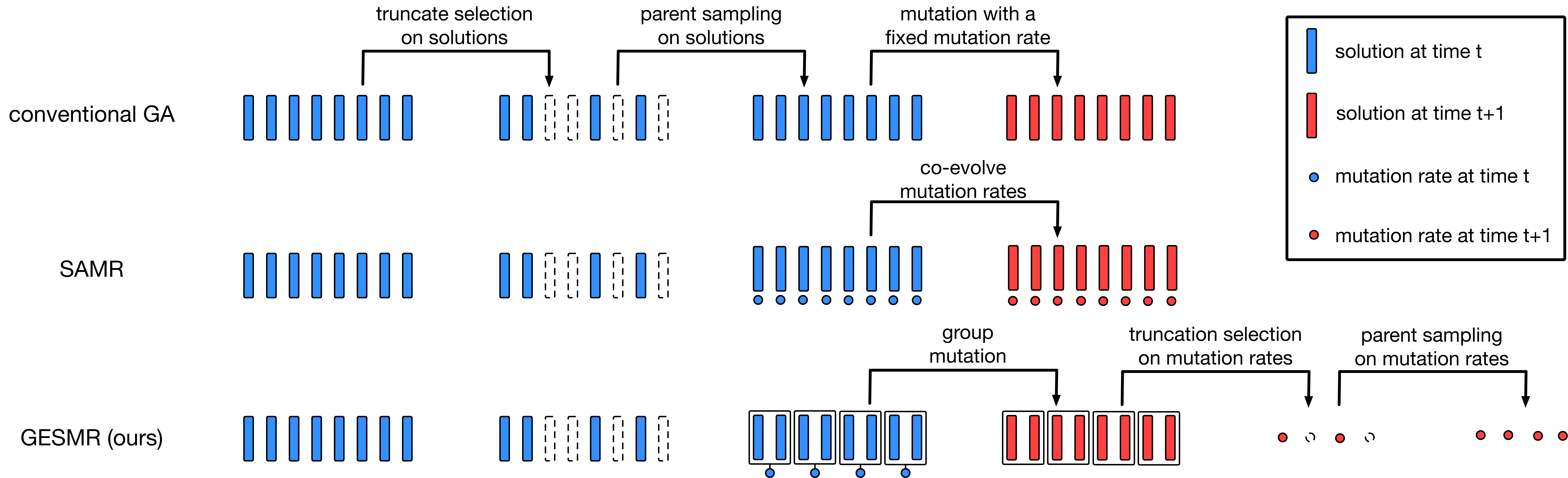}
    \caption{
    Comparison of GESMR against a fixed MR GA and SAMR.
    Fixed MR GA only evolves the solution with a given MR.
    SAMR evolves pairs of solutions and MRs.
    GESMR co-evolves a population of solutions 
    and a population of MRs separately.
    Each MR is assigned a group 
    and the MRs are evolved using the best function value 
    gain in the MR's corresponding group.
    }
    \label{fig:viz_algo}
\end{figure*}
\label{sec:method}
This section first provides the formal problem definition,
a discussion of the general class of genetic algorithms,
and then briefly describes a previous adaptive mutation rate (MR)
method and its associated vanishing MR problem.
Finally this section proposes the Group Elite Selection of Mutation Rates (GESMR) 
algorithm that addresses this problem with better performance and almost no
\linebreak
extra overhead.

\subsection{Problem Formulation}
\label{sec:formulation}
Consider the general optimization problem where the goal is to find the best decision variable $x^* \in \mathbb{R}^d$ that minimizes a target function $f$
(e.g. the negative fitness function in the genetic algorithm literature). 
The objective is therefore
\begin{equation}
    \argmin_{x \in \mathbb{R}^d} f(x).
    \label{eq:objective}
\end{equation}

\subsection{Genetic Algorithms and the Mutation Rate}
A genetic algorithm (GA) evolves a population of $N+1$ candidate solutions/individuals $x_0, \dots, x_N$ 
over time that progressively minimize the objective in Eq.~\ref{eq:objective}.
At each evolution time step $t$, the current population is 
$\{x^{(t)}_i\}_{i=0}^{N}$. 

To produce the next generation, a GA consists of 
1) selection of individuals, 
2) mutation of individuals, and 
3) crossover of individuals.

The common truncation selection method with one elite is used in this paper.
Truncation selection creates a new set of $N+1$ solutions
by keeping the single best ``elite" solution from the population (known as \emph{elitism})
and uniformly sampling the rest of the $N$ solutions
from the top $\eta_x$ portion of the population with replacement
(better solution has lower $f(x)$ value)
\cite{Such2017-rn}.



Since it is a common way to mutate a continuous genotype $x$ \cite{Such2017-rn},
the Gaussian mutation operator $M: \mathbb{R}^d \rightarrow \mathbb{R}^d$ is used, 
which produces $x'$ with
\begin{equation}
    x' \sim M(x; \sigma) \triangleq x + \sigma \epsilon, ~~\text{and}~~ \epsilon \sim \mathcal{N}(0, I).
\end{equation}
where $\mathcal{N}(0, I)$ denotes a standard multi-variate normal distribution in $\mathbb{R}^d$.
$\sigma \in \mathbb{R}_{\geq0}$ represents the \textbf{mutation rate} (MR), which constrains how different $x'$ could be from $x$. 

Crossover is used to mix information between solutions,
essentially allowing traits to be transferred to another solution.
For the sake of simplicity and to isolate the mutation operator, 
which is the main focus of this work,
no crossover operator is used
since crossover is not a necessary mechanism in GAs
\cite{Such2017-rn}.

For conventional GA algorithms, a fixed MR is chosen a priori based on the user's preference or prior knowledge.
Clearly, a too small $\sigma$ will slow down evolution
and a too large $\sigma$ will tend towards random search, a tuned $\sigma$ is needed.
It has also been shown that the optimal $\sigma$ changes over the course of evolution, 
e.g. a small $\sigma$ is often needed to ``fine tune" 
the solutions at the end of evolution \cite{Cervantes2009-me}.
As a result, the adaptive MR field studies how to dynamically adapt this $\sigma$ for faster learning and better convergence.
Among previous adaptive MR methods, a well-known and commonly used method 
is the self-adaptation of MR (SAMR)
\cite{Kramer2010-ia, Aleti2016-wn, Back1992-al, Gomez2004-sh, Thierens2002-dq}. 
This method attaches to each solution $x_i$ its own MR, $\sigma_i$.
These pairs $\{(x_i, \sigma_i)\}$ are then evolved, by selection on the pairs
and mutating the $x_i$ using $\sigma_i$ and mutating $\sigma_i$ using an external fixed meta MR $\tau$.

In practice, a well-known drawback of SAMR is that the MRs produced could 
prematurely converge to zero over time \cite{Rudolph2001-zf, Clune2008-lh, Glickman2000-mr}, 
which is referred to here as the \textbf{vanishing mutation rate problem} (VMRP).
One might try to simply clip the MR to a lower bound,
but a single lower bound that maintains exploration early on while still
allowing for fine tuning later may not exist \cite{Cervantes2009-me}.
Therefore, there exists a need for a better adaptive MR strategy.

\subsection{Group Elite Selection of Mutation Rates}
\label{sec:algorithm}

\begin{algorithm}[tb]
\caption{One step of GESMR}
\label{alg:algo}
\textbf{Input}: current solutions $\{x^{(t)}_i\}_{i=0}^N$, current mutation rates $\{\sigma^{(t)}_k\}_{k=1}^K$, the selection rates $\eta_x, \eta_\sigma$, and the meta mutation rate, $\tau$.\\
\textbf{Output}: next generation of solutions $\{x^{(t+1)}_i\}_{i=0}^N$ and mutation rates $\{\sigma^{(t+1)}_k\}_{k=1}^K$.
\begin{algorithmic}[1] 
\STATE // 1. Evolve the solutions
\STATE $\{\hat{x}^{(t)}_i\}_{i=0}^N \gets$ sort $\{x^{(t)}_i\}_{i=0}^N$ with ascending $f(\hat{x}^{(t)}_i)$
\STATE Generate $\{\tilde{x}_i^{(t)}\}_{i=0}^N$ according to 
Eq.~\ref{eq:select_solutions}
\COMMENT{Selection}
\STATE Generate $\{x_i^{(t+1)}\}_{i=0}^N$ according to Eq.~\ref{eq:mutate_solutions}\COMMENT{Mutation}
\STATE // 2. Evolve the mutation rates
\STATE Calculate $\Delta_k^{(t)}$ according to Eq.~\ref{eq:delta_algo}
\COMMENT{MR worth}
\STATE $\{\hat{\sigma}_k^{(t)}\}_{k=1}^K \gets$ sort $\{\sigma_k^{(t)}\}_{k=1}^K$ with ascending 
$\Delta_k^{(t)}$
\STATE Generate $\{\tilde{\sigma}_k^{(t)}\}_{k=1}^K$ according to 
Eq.~\ref{eq:select_mrs}
\COMMENT{Selection}
\STATE Generate $\{\sigma_k^{(t+1)}\}_{k=1}^K$ according to Eq.~\ref{eq:mutate_mrs}\COMMENT{Mutation}

\RETURN $\{x^{(t+1)}_i\}_{i=1}^N$ and $\{\sigma^{(t+1)}_j\}_{j=1}^K$
\end{algorithmic}
\end{algorithm}

This section presents Group Elite Selection of Mutation Rates (GESMR),
to adapt MRs on the fly, along with empirical evidence that 
GESMR mitigates the VMRP and outperforms previous adaptive MR methods.
For visualization of GESMR, refer to Fig.~\ref{fig:viz_algo}.

GESMR keeps a set of $K$ positive scalar MRs $\{\sigma_k\}_{k=1}^K$, where $N \equiv 0 \pmod{K}$, and co-evolves them with the $N+1$ candidate solutions, so that the $\sigma$s do not decay to zero.

At each optimization step $t$, 
the current population, $\{x_i^{(t)}\}_{i=0}^N$ 
is first sorted in ascending order of $f(x_i^{(t)})$, 
giving $\{\hat{x}_i^{(t)}\}_{i=0}^N$.
Truncation selection with one elite is applied
to get the next generation parents,
$\{\tilde{x}_i^{(t)}\}_{i=0}^N$, with
\begin{equation}
    \tilde{x}_i^{(t)} =
    \begin{cases}
        \hat{x}_0^{(t)} \qquad &i=0\\
        \sim \mathcal{U}\{\hat{x}_0^{(t)}, \dots, \hat{x}_{m-1}^{(t)}\} \qquad &i=1,\dots,N
    \end{cases}
    \label{eq:select_solutions}
\end{equation}
and $m=\eta_xN$ (number of solutions for parent selection).

Then, the \emph{non-elite} solutions, 
$\{\tilde{x}_1^{(t)}\}_{i=1}^N$
are split into $K$ groups of equal size 
(i.e. each group has $N/K$ solutions)
and each group is assigned a different $\sigma_k$. 
Without loss of generality, $\sigma_k$ corresponds to 
$\{\tilde{x}^{(t)}_{(k-1)N/K+1}, \dots, \{\tilde{x}^{(t)}_{kN/K}\}$. 
To form the next generation,
each $\tilde{x}^{(t)}_i$ is then mutated according to its corresponding $\sigma_k$,
while the elite is unaltered:
\begin{equation}
x^{(t+1)}_i = 
\begin{cases}
    \tilde{x}^{(t)}_0 \qquad &i=0 \\
    \sim M(\tilde{x}^{(t)}_i; \sigma_{\floor{iK/N}})\qquad &i = 1, \dots, N
\end{cases}
\label{eq:mutate_solutions}
\end{equation}

After the next generation of $\{x^{(t+1)}_i\}_{i=0}^N$ are found, GESMR 
\linebreak
evolves the MRs, 
$\{\sigma_k\}_{k=1}^K$ 
using another separate but similar GA with one elite, truncation selection,
and a different mutation operator.

For each $\sigma_k$, its negative fitness is calculated by considering 
the \emph{best} change in function value it has produced:
\begin{equation}
    \Delta_k^{(t)}\triangleq\Delta(\sigma_k^{(t)}) = \min_{i=(k-1)N/K+1}^{kN/K} \big(f(x^{(t+1)}_i) - f(\tilde{x}^{(t)}_i)\big).
    \label{eq:delta_algo}
\end{equation}

First the MR population is sorted by this $\Delta_k^{(t)}$, producing
$\{ \hat{\sigma}_{k=1}^K \}$.
Truncation selection with one elite is applied to get the next generation parent MRs $\{ \sigma_k \}_{k=1}^K$ with
\begin{equation}
    \tilde{\sigma}_k^{(t)} =
    \begin{cases}
        \hat{\sigma}_1^{(t)} \qquad &k=1\\
        \sim \mathcal{U}\{\hat{\sigma}_1^{(t)}, \dots, \hat{\sigma}_l^{(t)}\} \qquad &k=2,\dots,K
    \end{cases}
    \label{eq:select_mrs}
\end{equation}
and $l = \eta_\sigma K$ (number of MRs for parent selection).
The mutation operator associated with the $\sigma$s is 
$$\sigma' \sim M_\sigma(\sigma;\tau) \triangleq \sigma \tau^{\epsilon} ~~\text{and}~~ \epsilon\sim\mathcal{U}(-1, 1)$$
where $\mathcal{U}(-1, 1)$ represents a continuous uniform distribution on $\mathbb{R}$
and $\tau$ represents a fixed meta mutation rate.

The next generation of MRs is produced by mutating the parent MRs, 
while the elite parent is unaltered:
\begin{equation}
\sigma^{(t+1)}_i = 
\begin{cases}
    \tilde{\sigma}^{(t)}_1 \qquad &i=1 \\
    \sim M_\sigma(\tilde{\sigma}^{(t)}_i; \tau)\qquad &i = 2, \dots, K
\end{cases}
\label{eq:mutate_mrs}
\end{equation}

One full step of GESMR is described in Alg.~\ref{alg:algo}.

The performance of GESMR depends on the number of groups, $K$. 
When $K=1$, GESMR recovers the fixed-MR method. 
When $K=N$, each solution aside from the elite is assigned a different MR,
a method reminiscent of the SAMR method. 
The experiment section shows that in practice
the optimal $K$ lies between $1$ and $N$,
and uncovers a heuristic on how to choose such a $K$.

\section{Experiment}
The experiments in this section are designed to answer the following questions:
\begin{enumerate}
    \item How does GESMR compare to other methods 
    in terms of the quality of function values found and 
    how quickly it converges to those values?
    \item Does SAMR suffer from the Vanishing Mutation Rate Problem (VMRP)?
    Does GESMR solve this problem, and can it produce MRs that are optimal in a long-term sense?
    \item What parts of GESMR are vital to its success?
    \item Why is GESMR more successful than SAMR?
    \item What is the optimal group size in GESMR and how much does this parameter matter?
    \item Does GESMR generalize to the high-dimensional loss landscapes
    of neuroevolution?
    \item Does GESMR generalize to neuroevolution for 
    reinforcement learning control tasks?
\end{enumerate}

\subsection{Comparison Algorithms}

For comparison, the following MR selection and adaptation algorithms
are evaluated in various optimization problems:
\begin{itemize}
    \item \textsuperscript{\textdagger}\textbf{OFMR}:
    Optimal fixed MR found with a grid search;
    \item \textsuperscript{\textdagger}\textbf{LAMR-$G$}:
    MR determined at every $G$ generations by 
    \linebreak
    ``looking ahead," that is, by running a grid search 
    multiple times and picking the MR that produces the best elite in 
    another evolution run (initialized with the current 
    population and
    run for $G$ generations);
    \item \textbf{FMR}: A fixed MR of $\sigma=0.01$;
    \item \textbf{1CMR} A fixed MR of $\sigma = 1/d$ \cite{Ochoa2002-ni};
    \item \textbf{15MR}:
    MR is doubled if the percentage of beneficial mutations 
    is above 1/5 in the current generation 
    and cut in half if not \cite{Rechenberg1978-np};
    \item \textbf{UCB/$R$}: The adaptive MR method proposed by \citet{Fialho2008-ud},
    implemented with a multi-armed bandit with $R$ arms (each corresponding to a different MR), 
    and sampling an arm every generation using the upper confidence bound algorithm \cite{Fialho2008-ud};
    \item \textbf{SAMR}:
    Self-adaptation of MR, where each solution is assigned
its own MR and evolved end-to-end;
    \item \textbf{GESMR}: 
    The method of Algorithm~\ref{alg:algo};
    \item \textbf{GESMR-AVG}: 
    The method of Algorithm~\ref{alg:algo}
    with the min in Eq.~\ref{eq:delta_algo} replaced with the mean;
    \item \textbf{GESMR-FIX}: 
    The method of Algorithm~\ref{alg:algo}
    with the MRs fixed to the initial population and not evolved further.
\end{itemize}
Details for the parameters of these algorithms are provided in Appendix~A.
The \textdagger represents that the algorithm
is an oracle using foresight 
(looking ahead of the current evolution step)
to determine the MR and should not be compared against directly.
Note that \textbf{LAMR-$G$} specifically uses foresight 
to determine the best MR for the next $G$ generations. 
With sufficiently large $G$, its MRs thus serve as an empirical estimate of the 
optimal long-term MRs at any point during evolution.

\subsection{Test Optimization Functions}
All algorithms are evaluated on common test functions:
Ackley, Griewank, Rastrigin, Rosenbrock, Sphere, and Linear 
\cite{Surjanovic2013-fl}. 
Definitions of these test functions are provided in Appendix~B.1.
Each function is evaluated for dimension
$d\in\{2, 10, 100, 1000\}$,
with the initial population sampled from 
$\mathcal{N}(\mathbf{0}, \mathbf{I})$ and 
$\mathcal{N}(\mathbf{0}, 10^2\mathbf{I})$ (referenced in table as std with values 1 and 10).
These functions were chosen
because they are common in the GA literature
and they span a diverse range of ruggedness 
for function landscapes \cite{Malan2009-jz}.
All results are averaged over five seeds.

Fig.~\ref{fig:fits_mrs_test} shows selected runs from this experiment,
displaying the elite function value and the average MR over generations.
The full list of final elite function values are reported in 
Table~1
in Appendix~B.2,
serving as a statistic on how good the final solution is.
The full list of average elite function values over all evolution iterations are reported in 
Table~2
in Appendix~B.2,
serving as a statistic on how quickly the algorithm converges to a good solution.
Mean squared error between the log MR of an algorithm and the log MR of LAMR-100 (averaged over generations) are reported in 
Table~3
in Appendix~B.2,
serving as a statistic on how close to optimal the MRs are. 
Additionally, all of the tables bold the statistically significant results which are computed by a t-test.

\begin{figure}
\centering
\includegraphics[width=\columnwidth]{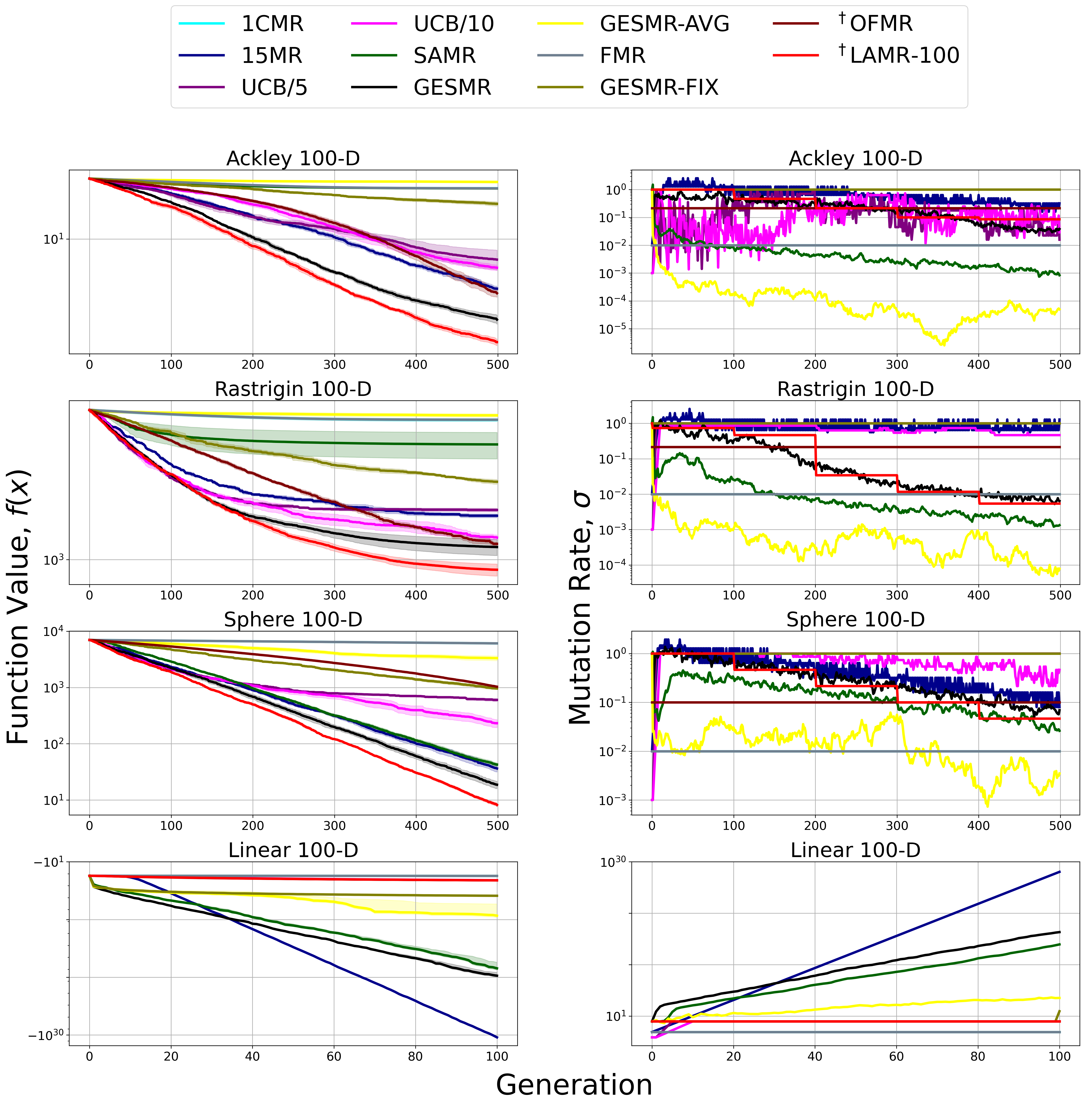}
\caption{
Elite function value and average mutation rate (MR)
over generations of evolution by different adaptive MR methods,
applied to four test optimization problems.
Notice GESMR outperforms other methods in function value
and is able to match its MR to the one from LAMR-$100$.
}
\label{fig:fits_mrs_test}
\end{figure}

To answer Question 1,
GESMR outperforms other methods, excluding the oracles,
in almost all domains both in terms of the final function value
and in terms of quickness of convergence to good values.

To answer Question 2,
SAMR only succeeds and matches the performance of LAMR
when the function landscape is relatively non-rugged, 
like in the Rosenbrock and Sphere functions.
In the rugged functions, SAMR consistently produces MRs
that are sub-optimal and smaller than those produced by even OFMR,
and thus also lags behind in elite function value during evolution.
Thus, SAMR struggles with the VMRP, as shown in previous work
\cite{Rudolph2001-zf, Clune2008-lh, Meyer-Nieberg2007-nz}.
However, GESMR overcomes this phenomenon
and surprisingly consistently \emph{matches its average MR 
to the long-term optimal MR} produced by LAMR-$100$
(i.e.\ red and black lines match in Fig.~\ref{fig:fits_mrs_test},
and GESMR has consistently the lowest error in 
Table~3 in Appendix~B.2).

The limitations of
of all methods except 15MR, SAMR, and 
\linebreak
GESMR
can be seen in the linear test function.
The optimal MR for this case is $\sigma\to\infty$,
but other methods are unable to approximate this result
because they limit themselves to an upper bound
(ex. UCB-$R$ is limited by the largest MR in its arms).
On the other hand, GESMR quickly keeps scaling up the MR until reaching a very large MR.
GESMR is also arbitrarily precise, fine tuning MRs with an evolutionary process.
In contrast, UCB-$R$ and the grid search methods constrain the MRs to a quantized range.

To answer Question 3, GESMR-AVG and GESMR-FIX 
were run as an ablation of GESMR, with the results
shown in Fig.~\ref{fig:fits_mrs_test} and 
Tables~1,~2,~3 in Appendix~B.2.
GESMR outperforms both of them, suggesting that the use of the best mutation statistic and the evolution of MRs are both vital to its success.

\subsection{Empirical Analysis of GESMR vs. SAMR}
\label{sec:empiricalanalysis}
To answer Question 4, 
two objectives for $\sigma$ are defined based on a
change of function value,
and these objectives are shown to be 
related to the GESMR-AVG, GESMR, and SAMR methods.
These objectives are then analyzed empirically (in this section) and theoretically (in Section~\ref{sec:theoreticalanalysis} to explain the
behavior of the algorithms.

Consider the \textbf{change in function value} of a mutation given a solution and an MR:
\begin{equation}
    \Delta(x, \sigma) \sim f(M(x;\sigma)) - f(x).
    \label{eq:delta_analysis}
\end{equation}
For simplicity, this variable will be denoted as $\Delta$.
Let $\akset{\Delta_q}_{q=1}^{N/K}$ represent
independently and identically distributed 
instances of $\Delta$ where $q$ indexes an individual within its group.
To minimize $f(x)$ in evolution, a $\sigma$ must be chosen to
minimize $\Delta(x,\sigma)$ in some capacity (denoted as an ``MR objective").
Consider two MR objectives
\begin{itemize}
    \item 
    \textbf{mean objective},
    $\sigma^*_\mu = \argmin_\sigma \EX_{x,\epsilon}[\Delta(x,\sigma)]$ and
    \item
    \textbf{outlier objective}, $\sigma^*_\text{min} = \argmin_\sigma \EX_{x,\epsilon}[\min_q \Delta_q(x,\sigma)]$.
\end{itemize}
The expectations in the objectives
are over $x$ sampled from the current population 
and the noise in the mutation operator, $\epsilon$.
For simplicity, these objectives are denoted as
$\argmin_\sigma \EX[\Delta]$ and 
\linebreak
$\argmin_\sigma \EX[\min_q\Delta_q]$, respectively.
The mean objective 
\linebreak
corresponds to the algorithm
GESMR-AVG, which selects $\sigma$s directly to minimize a
sample average of $\Delta$.
The outlier objective corresponds to the algorithm
GESMR, which selects $\sigma$s directly to minimize the 
\emph{best} (lowest-value) sample of $\akset{\Delta_q}$.
SAMR does not select $\sigma$s directly,
but rather selects $(x_i, \sigma_i)$ pairs to minimize $f(x_i)$.
However, because $x_i$ is produced using the parent of $\sigma_i$, 
SAMR also selects pairs $(x_i,\sigma_i)$ indirectly  
based on $\sigma_i$s that produce
non-deleterious mutations over generations consistently.
This mechanism is intuitively associated with the mean objective.

\begin{figure}[t]
\centering
\includegraphics[width=\columnwidth]{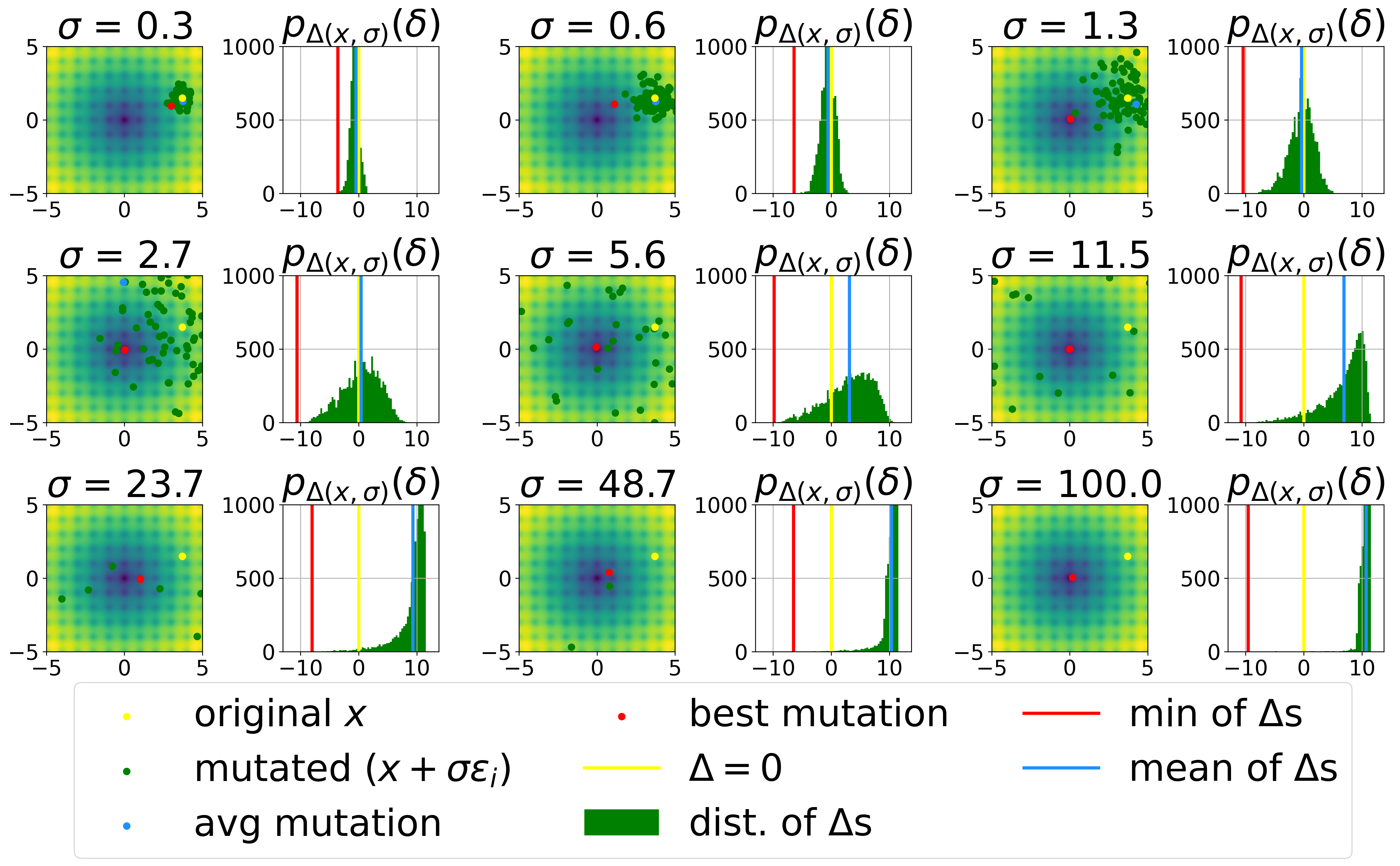}
\caption{
Visualization of mutations and the distribution of the 
change in function value from the mutations, $\Delta(x, \sigma)$ 
(defined in Eq.~\ref{eq:delta_analysis}),
for nine labeled mutation rates, $\sigma$,
at one point, $x$, on the 2-D Ackley function.
The left plots show an image representation of the 2-D function
landscape where lighter colors are higher values
and annotates the original solution and some mutated solutions.
The right plots show the empirical histogram of $\Delta(x, \sigma)$
and annotates the mean and minimum samples of this histogram.
Only moderate $\sigma$s are able to mutate to the global minimum.
}
\label{fig:landscape_mut}
\end{figure}

\begin{figure}[t]
\centering
\includegraphics[width=\columnwidth]{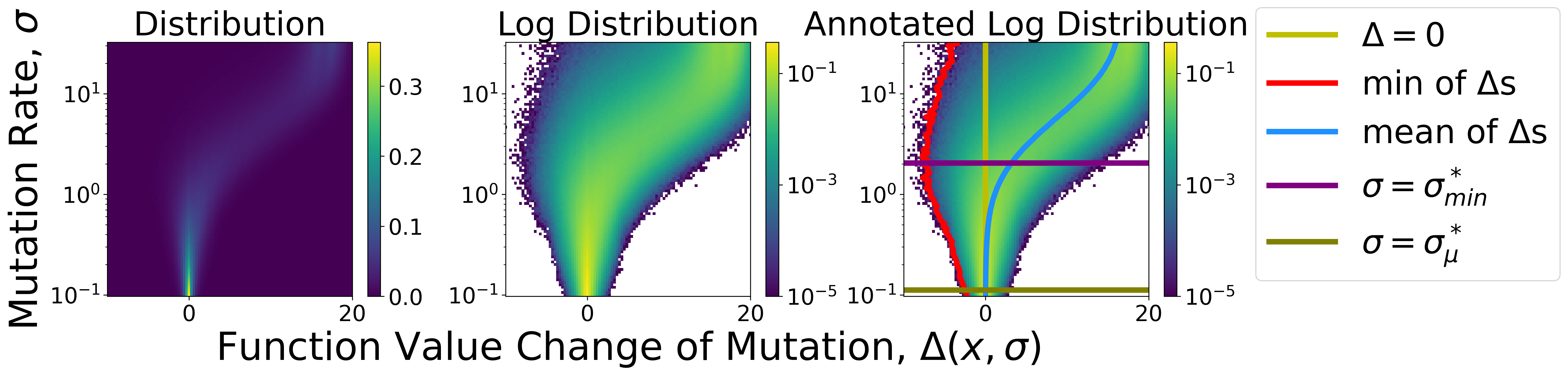}
\caption{
A representation of $\sigma$ versus 
$\Delta(x,\sigma)$ (defined in Eq.~\ref{eq:delta_analysis})
colored by the empirical probability density function, 
$p_{\Delta} (\delta; \sigma)$
and the respective log distribution for the 2-D Ackley function.
Many samples of $\Delta(x,\sigma)$ 
are generated from $x\sim\mathcal{N}(0, I)$,
and a logarithmic range of $\sigma$s, 
and put into bins of a $\sigma$-$\Delta$ grid,
colored by the number of samples the bin has.
Annotated are the $\sigma$ versus 
$\EX[\Delta; \sigma]$ (mean of $\Delta$s) and
$\EX[\min_q\Delta_q; \sigma]$ (min of $\Delta$s) curves,
and the $\sigma$s that minimize them.
Importantly, notice that $\sigma^*_\mu=0$ and $\sigma^*_\text{min}>0$.
}
\label{fig:curve_template}
\end{figure}

To analyze general 
function landscapes outside of evolution,
$x$ is either fixed to a point or sampled from a distribution, 
and many more samples for $\akset{\Delta_q}$ are used.
Fig.~\ref{fig:landscape_mut} shows
a histogram of samples from $\Delta$ and 
visualizes their respective mutations 
across values of $\sigma$ for a single $x$ in the Ackley 2-D function,
highlighting that the best mutation comes from a $\sigma$
that is not too small and not too large.
Fig.~\ref{fig:curve_template} represents this same information,
but sampling $x\sim \mathcal{N}(0,I)$,
for a continuous range of $\sigma$ as a 
visualization of the probability density function (PDF),
$p_{\Delta}(\delta; \sigma)$.
The sigma versus the mean objective and the 
outlier objective curves as well as
their optimal $\sigma$ solutions, 
$\sigma^*_\mu$ and $\sigma^*_\text{min}$
are shown over the PDF.
Fig.~\ref{fig:curves} displays this same plot
for several other test optimization problems.

\begin{figure}[t]
\centering
\includegraphics[width=\columnwidth]{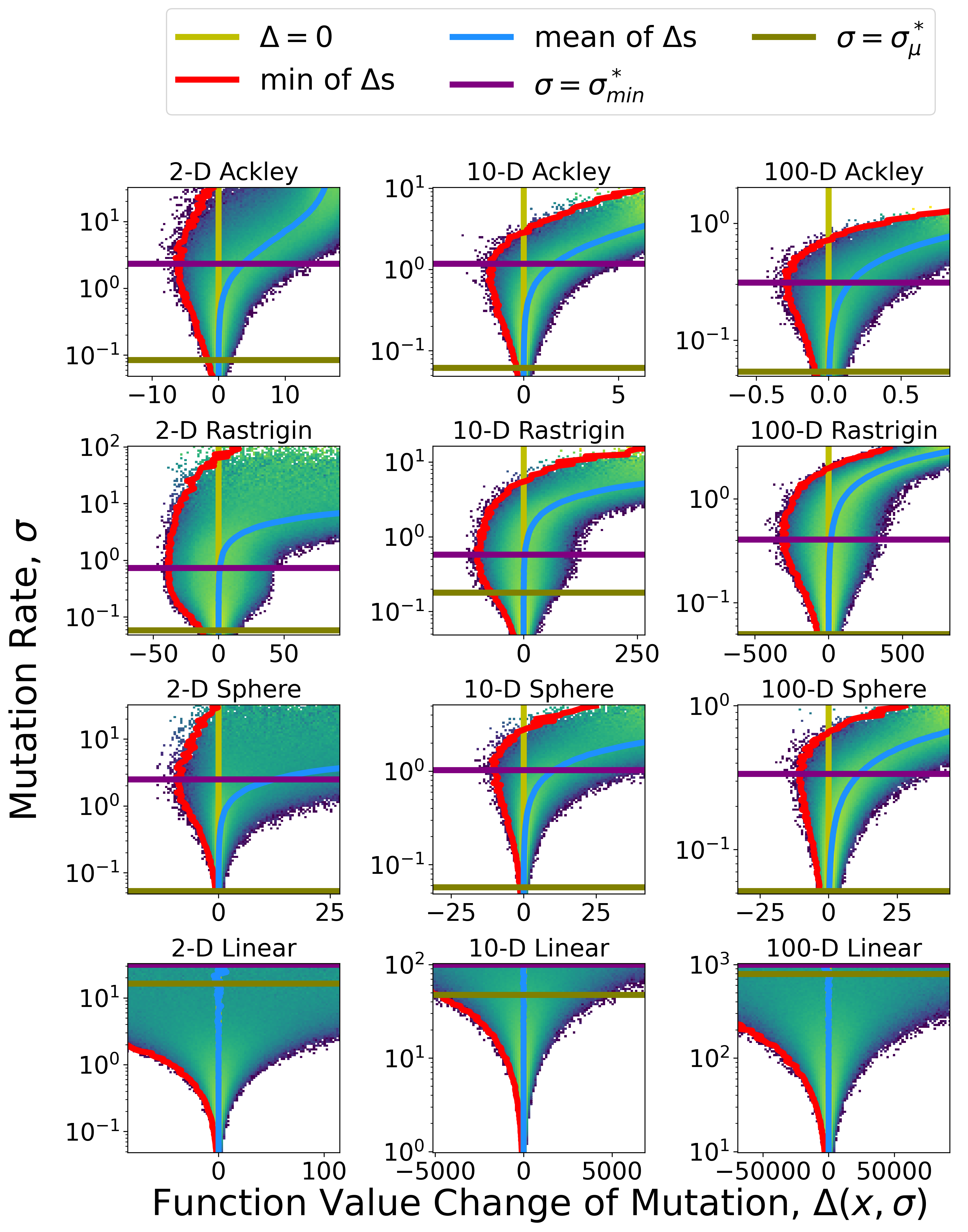}
\caption{
A representation of the $\sigma$ versus 
$\Delta(x,\sigma)$(defined in Eq.~\ref{eq:delta_analysis})
colored by the empirical probability density function
$p_{\Delta} (\delta; \sigma)$, 
and the respective log distribution
for several different test optimization functions of different dimensionality.
Many samples of $\Delta(x,\sigma)$
are generated from $x\sim\mathcal{N}(0, I)$
and a logarithmic range of $\sigma$s
and put into bins of a $\sigma$ versus $\Delta$ 2-D grid,
colored by the number of samples the bin has.
Annotated are the $\sigma$ versus 
$\EX[\Delta; \sigma]$ (mean of $\Delta$s), 
$\EX[\min_q\Delta_q; \sigma]$ (min of $\Delta$s),
and $\EX[\max_q\Delta_q; \sigma]$ (max of $\Delta$s) curves,
and the optimal $\sigma$ that minimizes the first two curves.
All curves show that $\sigma^*_\mu\to 0$ and $\sigma^*_\text{min}>0$.
}
\label{fig:curves}
\end{figure}

As Fig.~\ref{fig:curves} shows
$\EX[\Delta]$ often increases monotonically with $\sigma$. As a result, the optimal MR tends to go to zero, i.e. $\sigma^*_\mu \to 0$.
Interestingly, $\EX[\min_q\Delta_q]$ is zero for $\sigma=0$, and \emph{decreases} monotonically as $\sigma$ increases 
until $\sigma=\sigma^*_\text{min}$,
and then increases monotonically with $\sigma$, leading to
$\sigma^*_\text{min} \not\to 0$.
These behaviors hold true for all landscapes tested, except
for the non-rugged linear landscape.

These results answer Question~4 by showing empirically that
GESMR targets higher MRs than SAMR in many problems,
demonstrating that it has the capacity to mitigate the VMRP.
Theoretical analysis of GESMR and SAMR further
grounds this empirical finding to prove that
GESMR will always avoid a fully vanishing MR.


\subsection{Theoretical Analysis of GESMR vs. SAMR}
\label{sec:theoreticalanalysis}
In this section, the behavior of the mean and outlier MR objectives are analyzed as $\sigma\to0$ and $\sigma\to \infty$ \textit{during} evolution.
The current population of $x$ is assumed to be already \textit{partially} optimized, i.e., 
better than those of random search (which is the initialization). 
Partial optimization also means that the evolution has not yet converged, and thus 
the gradient of the function at the solutions, $\nabla f(x)$, is nonzero.

Assume $\sigma\to\infty$, fully and randomly exploring the solution space without exploiting the current solutions.
Then, \linebreak 
$\EX[\Delta]=\EX_{x'}[f(x')]-\EX_x[f(x)]$
(first expectation is over all mutants, $x'$)
becomes a constant only based on the function landscape and the distribution of $x$.
$\EX[\min_q\Delta_q]$ becomes, by definition, random search of the function landscape.
Since $x$ is already partially optimized,
random search must yield a strictly worse expected solution than $x$.
So, $\EX[\Delta] > \EX[\min_q\Delta_q] > 0$, and thus 
\emph{both MR objectives are positive}.

Assume $\sigma=0$ (i.e. no mutation), fully exploiting the current solution without exploring the solution space.
Then, \emph{both MR objectives vanish} as $\EX[\Delta] = \EX[\min_q\Delta_q] = 0$.

The most interesting case is when $\sigma$ is small but not zero, i.e.\ $0<\sigma<\sigma_c$.
For a sufficiently small $\sigma_c$ the function landscape can be approximated as linear with $f(M(x;\sigma)) \approx f(x)+\sigma\epsilon^T\nabla f(x)$.
Then, $\Delta(x,\sigma) = f(M(x;\sigma))-f(x)= \sigma\epsilon^T\nabla f(x)$.
Since $\epsilon \sim \mathcal{N}(0, I)$, it follows that
$\Delta(x,\sigma) \sim \mathcal{N}(0, \sigma^2\Vert\nabla f(x)\Vert^2)$,
which leads to
$\EX[\Delta]=0$. 
A further useful constraint is provided by Theorem~\ref{the:mingaus}:
\begin{theorem}
\label{the:mingaus}
Let $Z_\sigma^{(1)},\dots, Z_\sigma^{(q)} \sim \text{iid } \mathcal{N}(0,\sigma^2)$.\\
\hspace*{3ex}\text{If~} $Y_\sigma = \min(Z_\sigma^{(1)},\dots, Z_\sigma^{(q)})$,
\text{then} $\EX[Y_\sigma] = \sigma \EX[Y_{\sigma=1}]$ \text{with}\\
\hspace*{3ex}$\EX[Y_{\sigma=1}]<0$.
\end{theorem}
\begin{proof}
By definition, $f_z(z) = \phi(z/\sigma)$ and $F_z(z) = \Phi(z/\sigma)$.
Then,\\[-3.5ex]
\begin{align*}
    P(Y_\sigma\leq y) = 1-P(Y_\sigma\geq y) &= 1-P(Z_\sigma^{(1)}\geq y, \dots, Z_\sigma^{(q)}\geq y)\\
    &=1-(1-\Phi(y/\sigma))^q\\
    f_{Y_\sigma}(y)&=\frac{1}{\sigma}q(1-\Phi(y/\sigma))^{q-1}\phi(y/\sigma)\\
    \EX[Y_\sigma]&=\int_y \frac{y}{\sigma}q(1-\Phi(y/\sigma))^{q-1}\phi(y/\sigma)\\
    &=\sigma \int_y y q(1-\Phi(y))^{q-1}\phi(y)\\
    &=\sigma \EX[Y_{\sigma=1}].
\end{align*}
{\indent In addition, $\EX[Y_{\sigma=1}]<0$ because $Y_{\sigma=1}$ is the minimum of $q>1$}
{\indent zero-mean standard normal random variables.}
\qedhere
\end{proof}
\noindent
By Theorem~\ref{the:mingaus},
$\EX[\min_q\Delta_q]\propto \sigma \Vert\nabla f(x)\Vert <0$. Thus,
in this range of $\sigma$, \emph{the outlier objective \textit{decreases} linearly as $\sigma$ increases, while the mean objective still vanishes}.

Using these three cases, consider the MR objectives
as $\sigma$ varies from $0$ to $\infty$.
$\EX[\Delta]$ starts at $0$ and takes a
theoretically unknown (but empirically monotonic) path
to a positive value.
$\EX[\min_q\Delta_q]$ starts at $0$, 
\textit{decreases} to below $0$ until a certain $\sigma_c$, 
then takes a theoretically unknown (but empirically monotonic) path
to a positive value.
This theoretical analysis guarantees that $\sigma^*_\text{min}>0$,
a condition that cannot be put on $\sigma^*_\mu$.

Thus, this section and Section~\ref{sec:empiricalanalysis} empirically and theoretically answer Question~4, i.e.\ explain \emph{why}
GESMR-AVG and SAMR often suffer from the VMRP in rugged landscapes,
and \emph{how} GESMR overcomes this limitation.
In short, GESMR-AVG and SAMR assume that $\sigma$ produce non-deleterious mutations consistently,
whereas most mutations are actually deleterious \cite{Clune2008-lh}.
This condition is possible only if $\sigma\to 0$,
which GESMR incorporates
into the algorithm itself.




\subsection{Ablation on the Group Size Parameter}
To answer Question 5, and to evaluate the optimal number of
\linebreak
groups, $K$,
evolution was run on the Ackley, Griewank, Rosenbrock, and Sphere functions 
with $d=100$ and $K$ equal to all
factors of $N$ for various values of $N$.
It turns out that 
if the number of groups is
too small, i.e. $K\to 1$, or too big, i.e. $K\to N$, 
the performance drops very fast (Fig.~\ref{fig:group_size}).
In general, $K=\sqrt{N}$ is a reasonable value,
but as $N\to\infty$, 
the optimal $K\to N^{3/4}$.
This finding suggests that the number-of-groups
hyperparameter can be set according to $N$ and does not need tuning.

\begin{figure}
\centering
\includegraphics[width=\columnwidth]{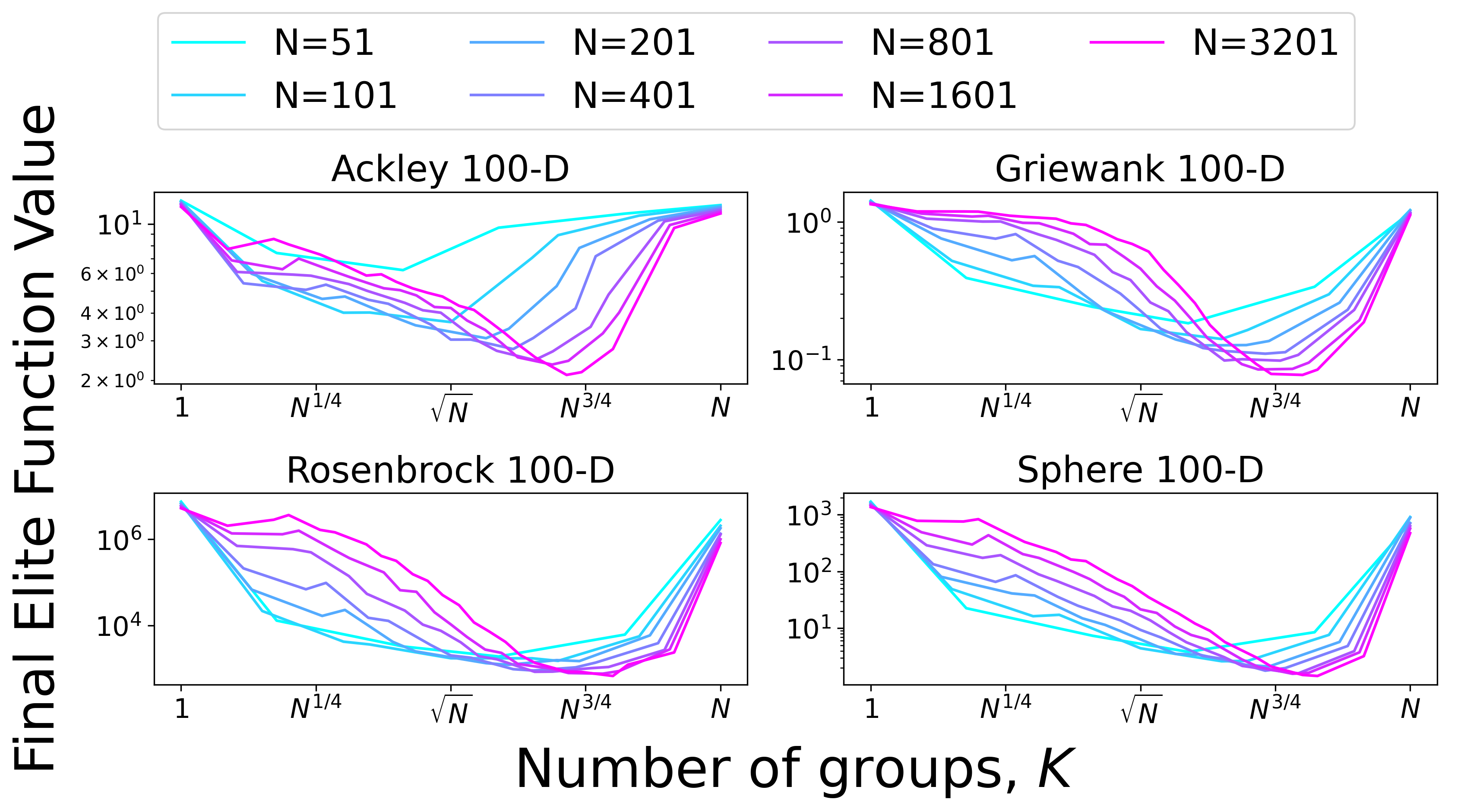}
\caption{
Elite final function value of 
GESMR versus the number of groups, $K$,
as the population size, $N$ increases
in the Ackley 100-D function.
As $N\to\infty$, the optimal $K\to N^{3/4}$, 
suggesting $K$ does not need tuning.
}
\label{fig:group_size}
\end{figure}

\subsection{Neuroevolution for Image Classification}

\begin{figure}
\centering
\includegraphics[width=\columnwidth]{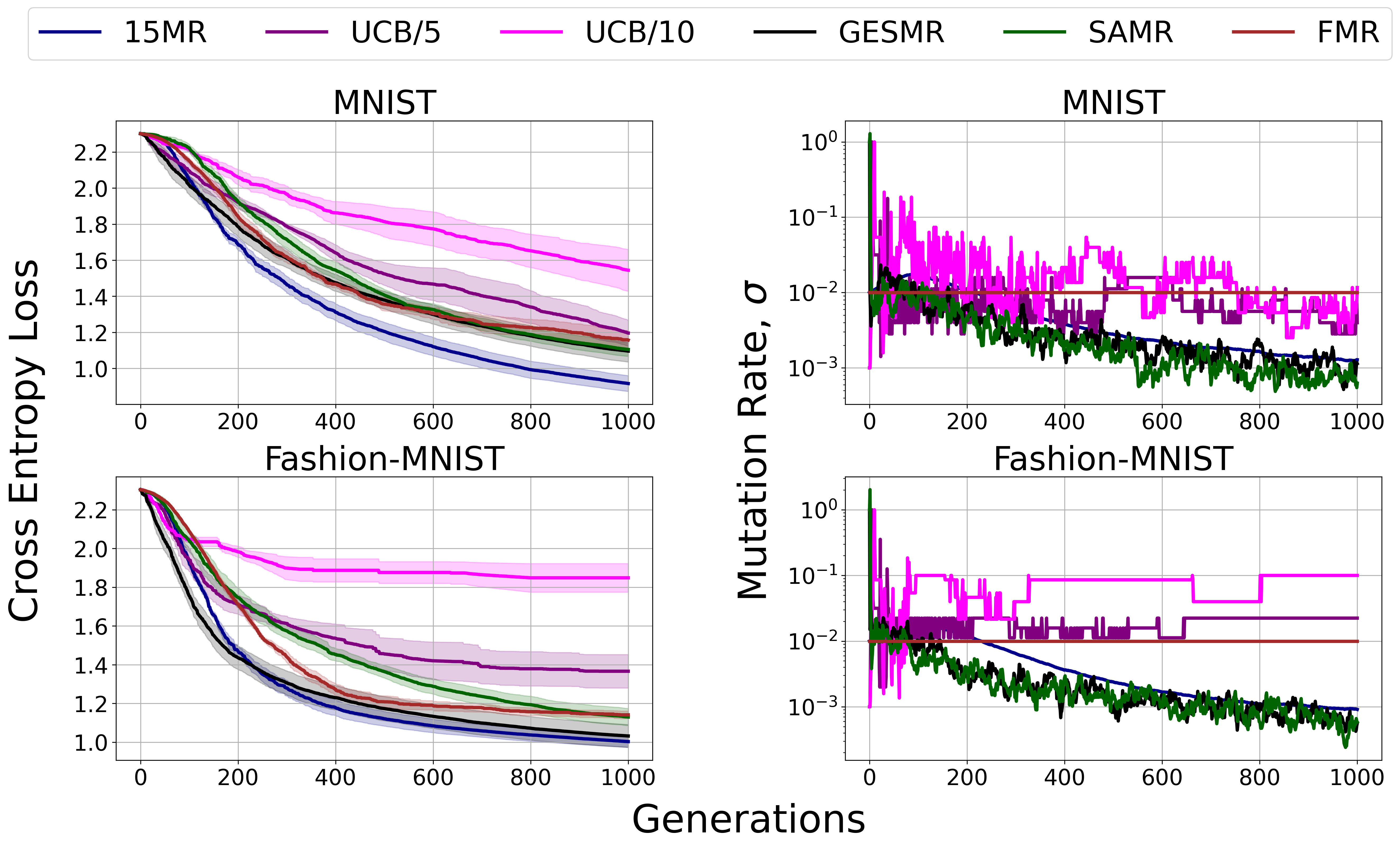}
\caption{
Elite function value and average mutation rate
(for different mutation rate control strategies)
versus generations of neuroevolution
applied to image classification in MNIST and Fashion-MNIST.
GESMR outperforms most methods except 15MR, which appears to be an especially good fit for this problem.
}
\label{fig:fits_mrs_mnist}
\end{figure}

To answer Question 6, the algorithms were 
run on the high dimensional loss landscapes of 
neuroevolution for image classification with the 
common MNIST and Fashion-MNIST datasets 
\cite{LeCun1998-wt, Xiao2017-kj}.
The details of the datasets, the NN architecture evolved, and the experimental setup
are provided in Appendix~C.
Each algorithm was run independently five times and the mean loss and the standard error measured.

GESMR outperforms all other methods, including FMR and
\linebreak
SAMR,
but does not beat 15MR (Fig.~\ref{fig:fits_mrs_mnist}).
Presumably, 15MR's hyperparameter of 1/5 is especially
suited to the MNIST loss landscapes but might have trouble generalizing to other problems,
like the test optimization problems and the reinforcement learning control problems.

\subsection{Neuroevolution for Reinforcement Learning}
Reinforcement learning (RL) tasks are amenable to the neuroevolution approach
because the approach tolerates long time-horizon rewards well \cite{Salimans2017-ji, Such2017-rn}.
To answer Question 7, the algorithms were evaluated 
on four common RL control tasks:
CartPole, Pendulum, Acrobot, and MountainCar
\cite{Brockman2016-bx}. In all these tasks, a controller maps the robot's input observations
to either continuous or discrete actions to maximize a cumulative reward.
The details of these environments, the neural architecture evolved, and the experimental setup
are provided in the Appendix~D.
Each algorithm was run independently five times and the mean and standard error of performance was measured.

The results are shown in
Fig.~\ref{fig:fits_mrs_rl}
in the Appendix~D.
GESMR generally outperformed other methods including 
the baseline fixed MR and SAMR.
Presumably, GESMR fails in MountainCar
because the reward signal is very sparse
(zero rewards provide no way to appropriately select for MRs).

\subsection{Comparison against CMA-ES}
CMA-ES is not a pure adaptive MR GA method: It stores a covariance matrix
to control the spread of the population, rather than storing a single MR \cite{Hansen2016-qa}.
This matrix grows quadratically with the solution vector length.
However, CMA-ES still provides an interesting comparison 
given a fixed computational budget.
Fig.~\ref{fig:gesmr_vs_cmaes} shows that GESMR outperforms
CMA-ES significantly in four of the most challenging test optimization problems,
even 
\linebreak
though 
CMA-ES uses much more memory (quadratic in the solution space).
Thus, not only does GESMR scale to higher dimensional problems,
it also outperforms CMA-ES when both are given the same running time.

\begin{figure}
\centering
\includegraphics[width=\columnwidth]{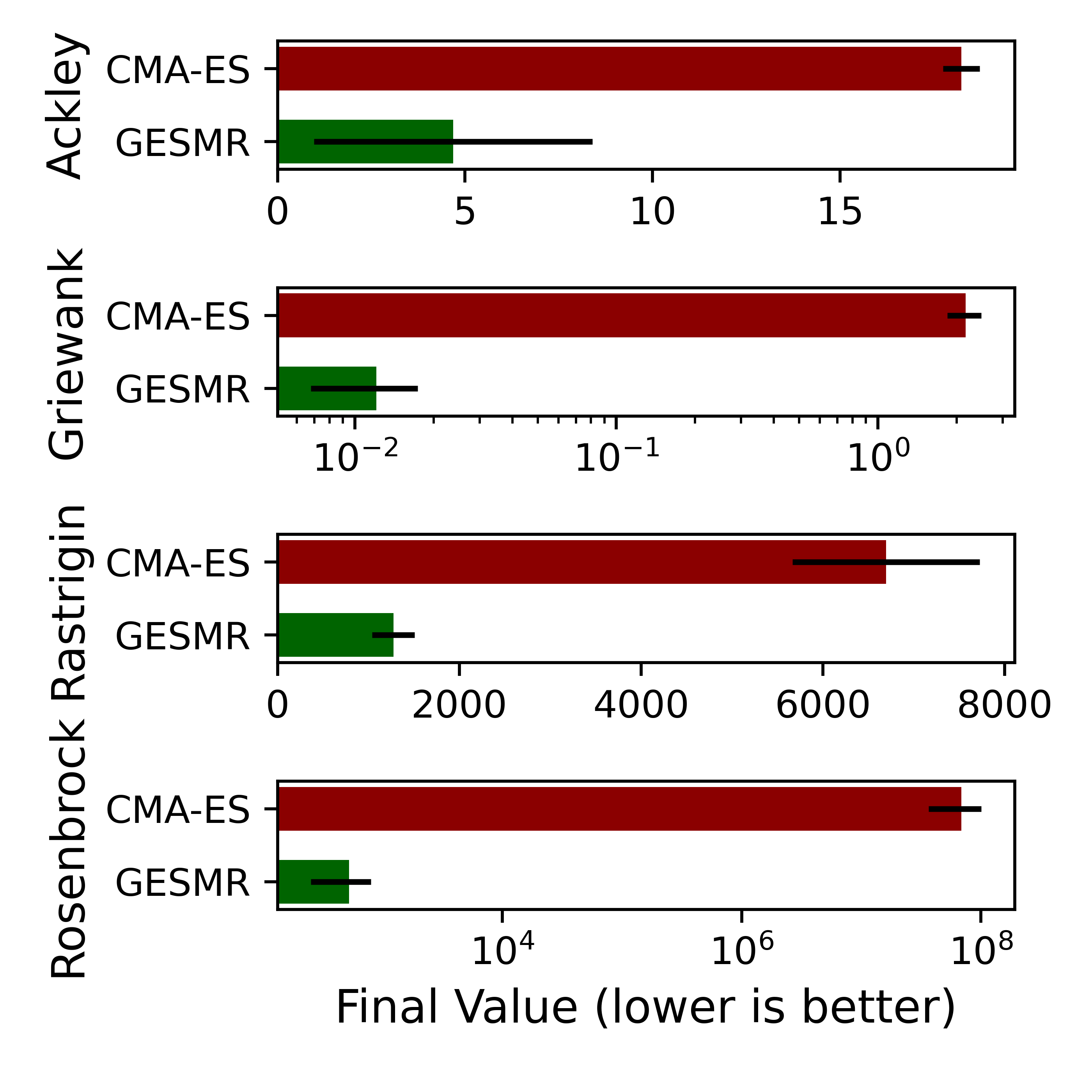}
\caption{
Elite final function values of GESMR and CMA-ES on
four challenging problems in the 100-D solution space with a fixed computational budget.
Whereas CMA-ES is only able to complete 50 generations, GESMR is completes approximately 1000, resulting in an order of magnitude better values.
}
\label{fig:gesmr_vs_cmaes}
\end{figure}

\section{Conclusion}

In this paper, a novel and simple adaptive mutation rate (MR)
\linebreak
method,
group elite selection mutation rate (GESMR),
was proposed to mitigate the vanishing mutation rate problem (VMRP), 
along with empirical analysis that grounds its success over
self-adaptation of mutation rates (SAMR).
Comprehensive experiment results showed that GESMR
outperforms previous adaptive MR methods
in final value and convergence speed.
GESMR also consistently matches its MRs to the 
empirically estimated long-term optimal MR.
Thus, this work provides the next step in designing self-adaptive machine learning algorithms.

\bibliographystyle{ACM-Reference-Format}
\bibliography{sample-base}

\clearpage
\appendix

\section{General Experiment Setup}
All algorithms for all experiments 
(except the group size ablation experiment)
are run with a population size of $N+1=101$.
The test optimization problems are run for 
$T\in\akset{100, 300, 1000, 2500}$ generations with
problem dimensionality $d\in\akset{2, 30, 100, 1000}$, respectively.
The Linear function is always only run for $T=100$ generations.
The MNIST/Fashion-MNIST experiments are run for $T=1000$
generations.
All reinforcement learning experiments are run for $T=100$
generations.

\textbf{OFMR}
finds the optimal fixed MR
using a grid search over a logarithmic range
of ten MRs ranging from \num{1e-3} to \num{1}.
For each MR in the grid search, an entire evolution is run to evaluate it.
The MR whose evolution provides the best final elite function value is picked as
the optimal fixed MR, and another fixed MR
evolution is run with this MR value.

\textbf{LAMR-$G$}
changes the MR every $G$ generations,
and picks the MR according to a grid search
over a logarithmic range of 10 MRs ranging from 
\num{1e-3} to \num{1}.
For each MR in the grid search, 
the current population is used to initialize another evolution run
that ooks ahead for $G$ generations.
The MR whose evolution provides the best final elite is
used for the next $G$ generations in the main evolution run.
In this way, LAMR-$G$ is able to adapt MRs for the long-term
by directly looking ahead $G$ generations and picking an that MR performs the best.

\textbf{FMR} sets the MR to a fixed \num{1e-2}, 
as is commonly done when the user is left to define an MR. 

\textbf{1CMR} sets the MR to a fixed $1/d$
where $d$ is the dimensionality of the solution space
\cite{Ochoa2002-ni}.
The goal is to search carefully in problems with high dimensionality
and explore more in problems with low dimensionality.

\textbf{15MR} 
starts with the MR equal to \num{1e-2} and 
adapts MRs based on the percentage of beneficial mutations
in the current generation (i.e.\ those that result in a negative function value change)
If the percentage is greater than 1/5, the MR is doubled,
else it is cut in half.
This factor of two is chosen to match the meta-MRs in SAMR and GESMR
in order to compare adaptability fairly between methods.

\textbf{UCB/$R$}
creates a multi-armed bandit problem with $R$ arms
corresponding to MRs that are spaced logarithmically between
\linebreak
\num{1e-3} and 1.
The upper confidence bound (UCB) algorithm is utilized to solve the problem.
At each generation, an MR is sampled from UCB;
the reward that is reported back is the \textit{best} (lowest)
change in function value from mutations
for the current generation.

With \textbf{SAMR},
solutions are paired up with MRs spaced logarithmically between \num{1e-3} and \num{1e3}.
The solutions are mutated according to their assigned MR
and the MRs are mutated with the same
equation as with GESMR, using the meta-MR $\tau=2$.

With \textbf{GESMR},
the population of MRs are initialized
by spacing them logarithmically between \num{1e-3} and \num{1e3}.
They are mutated using the meta-MR $\tau=2$.

\begin{table*}[t]
\centering

\scalebox{0.7}{

\begin{tabular}{c | r r | r r r r r r r r r r r}
 &
 \multicolumn{1}{c}{Dim}&
 \multicolumn{1}{c}{Std}&
 \multicolumn{1}{c}{\textbf{\textsuperscript{\textdagger}OFMR}}&
 \multicolumn{1}{c}{\textbf{\textsuperscript{\textdagger}LAMR-100}}&
 \multicolumn{1}{c}{\textbf{FMR}}&
 \multicolumn{1}{c}{\textbf{1CMR}}&
 \multicolumn{1}{c}{\textbf{15MR}}&
 \multicolumn{1}{c}{\textbf{UCB/5}}&
 \multicolumn{1}{c}{\textbf{UCB/10}}&
 \multicolumn{1}{c}{\textbf{SAMR}}&
 \multicolumn{1}{c}{\textbf{GESMR}}&
 \multicolumn{1}{c}{\textbf{GESMR-AVG}}&
 \multicolumn{1}{c}{\textbf{GESMR-FIX}}
 \\
\toprule

\multirow{8}{*}{\rotatebox[origin=c]{90}{Ackley}}
 &\multirow{2}{*}{2}
 & 1
 & 0.0 & 0.0 & 0.0 & 0.0 & 0.0 & 0.0 & 0.0 & \textbf{0.0} & 0.0 & 0.0 & 0.0\\
&& 10
 & 0.5 & 0.5 & 1.5 & 0.0 & 0.1 & 0.1 & 0.0 & \textbf{0.0} & 0.0 & 0.0 & 0.0\\
 &\multirow{2}{*}{30}
 & 1
 & 1.2 & 0.0 & 2.8 & 2.7 & \textbf{0.2} & 2.2 & 2.7 & 2.4 & 0.8 & 3.2 & 3.0\\
&& 10
 & 2.7 & 1.4 & 15.2 & 15.3 & 3.1 & 4.9 & 3.5 & 11.0 & *\textbf{1.0} & 15.6 & 6.6\\
 &\multirow{2}{*}{100}
 & 1
 & 3.0 & 2.5 & 3.1 & 3.1 & \textbf{2.4} & 3.1 & 2.9 & 2.9 & 2.7 & 3.9 & 3.5\\
&& 10
 & 4.1 & 2.3 & 16.3 & 16.3 & *\textbf{2.8} & 6.2 & 4.0 & 16.3 & 3.6 & 17.2 & 12.0\\
 &\multirow{2}{*}{1000}
 & 1
 & 3.6 & 3.5 & 3.7 & 4.6 & 3.5 & 3.9 & 4.0 & 3.6 & \textbf{3.5} & 4.9 & 4.6\\
&& 10
 & 12.0 & 15.9 & 17.4 & 18.2 & 17.6 & 17.8 & 17.3 & 17.3 & \textbf{17.1} & 18.3 & 18.1\\
\hline
\multirow{8}{*}{\rotatebox[origin=c]{90}{Griewank}}
 &\multirow{2}{*}{2}
 & 1
 & 0.0 & 0.0 & 0.0 & 0.0 & 0.0 & 0.0 & 0.0 & \textbf{0.0} & 0.0 & 0.0 & 0.0\\
&& 10
 & 0.0 & 0.0 & 0.0 & 0.0 & 0.0 & 0.0 & 0.0 & 0.0 & 0.0 & 0.0 & \textbf{0.0}\\
 &\multirow{2}{*}{30}
 & 1
 & 0.0 & 0.0 & 0.1 & 0.0 & 0.0 & 0.0 & 0.0 & 0.0 & *\textbf{0.0} & 0.2 & 0.1\\
&& 10
 & 0.1 & 0.0 & 1.3 & 1.2 & 0.2 & 0.2 & 0.2 & 0.0 & \textbf{0.0} & 1.1 & 0.7\\
 &\multirow{2}{*}{100}
 & 1
 & 0.0 & 0.0 & 0.0 & 0.0 & 0.0 & 0.1 & 0.1 & 0.0 & *\textbf{0.0} & 0.3 & 0.2\\
&& 10
 & 0.2 & 0.0 & 2.3 & 2.3 & 0.0 & 0.3 & 0.3 & 0.0 & *\textbf{0.0} & 1.5 & 1.0\\
 &\multirow{2}{*}{1000}
 & 1
 & 0.2 & 0.2 & 0.2 & 1.0 & 0.2 & 0.6 & 0.6 & 0.3 & *\textbf{0.2} & 0.9 & 0.8\\
&& 10
 & 2.4 & 1.8 & 19.3 & 22.8 & 2.4 & 6.5 & 5.2 & 2.6 & *\textbf{2.2} & 20.3 & 10.5\\
\hline
\multirow{8}{*}{\rotatebox[origin=c]{90}{Rastrigin}}
 &\multirow{2}{*}{2}
 & 1
 & 0.0 & 0.0 & 0.0 & 0.0 & 0.0 & 0.0 & 0.0 & \textbf{0.0} & 0.0 & 0.0 & 0.0\\
&& 10
 & 0.0 & 0.2 & 6.2 & \textbf{0.1} & 1.0 & 0.2 & 0.1 & 1.3 & 0.2 & 2.8 & 0.2\\
 &\multirow{2}{*}{30}
 & 1
 & 27.8 & 26.3 & 27.6 & 34.7 & 27.3 & 173.2 & 166.3 & \textbf{26.9} & 28.3 & 90.4 & 62.2\\
&& 10
 & 210.4 & 113.5 & 1544.5 & 1539.1 & 320.9 & 302.1 & 306.3 & 1108.8 & *\textbf{150.0} & 1517.0 & 368.3\\
 &\multirow{2}{*}{100}
 & 1
 & 109.1 & 104.0 & 113.7 & 113.3 & 118.3 & 782.9 & 760.4 & \textbf{94.0} & 111.1 & 348.6 & 263.8\\
&& 10
 & 984.0 & 839.7 & 6949.1 & 6934.4 & 1612.9 & 1793.1 & 1229.1 & 4918.4 & \textbf{1149.7} & 7240.4 & 2590.7\\
 &\multirow{2}{*}{1000}
 & 1
 & 2001.4 & 1689.2 & 2781.4 & 6748.2 & 2113.8 & 4798.8 & 9735.7 & 2153.1 & *\textbf{1878.5} & 7107.0 & 6576.5\\
&& 10
 & 2.6e+04 & 2.9e+04 & 8.9e+04 & 9.5e+04 & \textbf{4.1e+04} & 7.8e+04 & 4.8e+04 & 8.9e+04 & 6.0e+04 & 9.6e+04 & 9.3e+04\\
\hline
\multirow{8}{*}{\rotatebox[origin=c]{90}{Rosenbrock}}
 &\multirow{2}{*}{2}
 & 1
 & 0.0 & 0.0 & 0.0 & 0.0 & 0.0 & 0.0 & 0.0 & 0.0 & \textbf{0.0} & 0.1 & 0.0\\
&& 10
 & 0.0 & 0.0 & 1.0 & 0.0 & 0.1 & 0.0 & 0.0 & 0.2 & \textbf{0.0} & 1.4 & 0.0\\
 &\multirow{2}{*}{30}
 & 1
 & 40.5 & 39.0 & 157.2 & 43.3 & 30.2 & 1027.3 & 1485.6 & \textbf{28.1} & 39.5 & 589.2 & 303.3\\
&& 10
 & 383.3 & 183.9 & 1.2e+07 & 4.4e+06 & 813.3 & 4371.1 & 742.3 & 581.5 & \textbf{199.0} & 1.6e+06 & 3.9e+04\\
 &\multirow{2}{*}{100}
 & 1
 & 146.8 & 113.1 & 179.6 & 184.9 & \textbf{112.9} & 191.0 & 6830.5 & 157.2 & 117.9 & 2817.8 & 1484.5\\
&& 10
 & 3171.6 & 290.3 & 6.5e+07 & 6.5e+07 & \textbf{517.8} & 6.2e+05 & 9.7e+04 & 945.9 & 943.1 & 1.5e+07 & 3.5e+05\\
 &\multirow{2}{*}{1000}
 & 1
 & 1.2e+04 & 8931.6 & 1.2e+04 & 2.1e+05 & 1.1e+04 & 2.9e+04 & 2.1e+04 & 1.3e+04 & \textbf{9698.8} & 1.7e+05 & 1.3e+05\\
&& 10
 & 2.0e+07 & 1.1e+07 & 1.4e+09 & 2.2e+09 & 1.8e+07 & 2.5e+08 & 9.6e+07 & 3.1e+07 & \textbf{1.4e+07} & 1.0e+09 & 4.9e+08\\
\hline
\multirow{8}{*}{\rotatebox[origin=c]{90}{Sphere}}
 &\multirow{2}{*}{2}
 & 1
 & 0.0 & 0.0 & 0.0 & 0.0 & 0.0 & 0.0 & 0.0 & 0.0 & 0.0 & \textbf{0.0} & 0.0\\
&& 10
 & 0.0 & 0.0 & 0.0 & 0.0 & 0.0 & 0.0 & 0.0 & \textbf{0.0} & 0.0 & 0.0 & 0.0\\
 &\multirow{2}{*}{30}
 & 1
 & 0.0 & 0.0 & 0.7 & 0.0 & 0.0 & 0.3 & 4.6 & 0.0 & \textbf{0.0} & 4.2 & 1.6\\
&& 10
 & 2.1 & 0.0 & 1297.7 & 797.9 & 0.3 & 35.8 & 15.4 & 0.0 & \textbf{0.0} & 500.6 & 30.9\\
 &\multirow{2}{*}{100}
 & 1
 & 0.1 & 0.0 & 0.1 & 0.1 & 0.0 & 0.0 & 0.3 & 0.0 & \textbf{0.0} & 17.6 & 7.7\\
&& 10
 & 9.6 & 0.0 & 5234.4 & 5234.5 & 0.2 & 409.6 & 37.2 & 0.2 & \textbf{0.0} & 1610.9 & 222.3\\
 &\multirow{2}{*}{1000}
 & 1
 & 57.1 & 31.8 & 74.8 & 731.6 & 55.3 & 93.5 & 86.1 & 65.9 & \textbf{45.1} & 665.5 & 555.1\\
&& 10
 & 5484.4 & 3158.5 & 7.3e+04 & 8.7e+04 & 5481.0 & 5.2e+04 & 2.4e+04 & 6540.3 & \textbf{4459.2} & 6.7e+04 & 3.8e+04\\
\hline
\multirow{8}{*}{\rotatebox[origin=c]{90}{Linear}}
 &\multirow{2}{*}{2}
 & 1
 & -1614.0 & -1654.6 & -53.4 & -842.3 & \textbf{-3.1e+29} & -1547.5 & -1497.4 & -2.7e+16 & -8.7e+18 & -1.6e+08 & -6.9e+05\\
&& 10
 & -1930.6 & -1969.2 & -393.6 & -1172.9 & \textbf{-3.1e+29} & -1887.1 & -1837.6 & -1.7e+16 & -8.7e+18 & -1.6e+08 & -6.9e+05\\
 &\multirow{2}{*}{30}
 & 1
 & -6173.2 & -6274.7 & -202.7 & -343.7 & \textbf{-1.2e+30} & -5995.3 & -5620.6 & -1.8e+18 & -2.8e+19 & -1.1e+08 & -2.7e+06\\
&& 10
 & -7389.3 & -7446.0 & -1485.7 & -1626.6 & \textbf{-1.2e+30} & -7277.5 & -6903.6 & -1.3e+18 & -2.8e+19 & -1.1e+08 & -2.7e+06\\
 &\multirow{2}{*}{100}
 & 1
 & -1.1e+04 & -1.1e+04 & -332.1 & -330.2 & \textbf{-2.5e+30} & -1.1e+04 & -1.0e+04 & -8.2e+17 & -1.2e+20 & -1.1e+10 & -4.9e+06\\
&& 10
 & -1.3e+04 & -1.3e+04 & -2346.8 & -2345.0 & \textbf{-2.5e+30} & -1.3e+04 & -1.2e+04 & -6.9e+18 & -1.2e+20 & -1.1e+10 & -5.0e+06\\
 &\multirow{2}{*}{1000}
 & 1
 & -2.6e+04 & -2.5e+04 & -759.0 & -536.5 & \textbf{-4.7e+30} & -2.4e+04 & -2.4e+04 & -2.2e+18 & -2.0e+20 & -1.2e+11 & -1.1e+07\\
&& 10
 & -3.0e+04 & -2.9e+04 & -5369.0 & -5146.5 & \textbf{-4.7e+30} & -2.9e+04 & -2.8e+04 & -2.4e+18 & -2.0e+20 & -1.2e+11 & -1.1e+07\\

\bottomrule

\end{tabular}

}

\caption{
A genetic algorithm's final elite function value,
on various functions and population initializations 
using different mutation rate control strategies.
This metric quantifies good the final solution
found by the GA is.
The results are averaged over 40 seeds.
The best value is shown in bold.
A statistical t-test is performed on the best method and if the resulting p-value
is less than $0.05$ versus \textit{all} other methods, 
the result is considered significant and 
shown with an asterisk (*) in front of it. 
Methods marked with \textdagger~are oracles for benchmark
and are not compared against because they use foresight during evolution.
GESMR outperforms previous methods on most tasks, often significantly.
}
\label{tbl::data_table}
\end{table*}

\begin{table*}[t]
\centering

\scalebox{0.7}{

\begin{tabular}{c | r r | r r r r r r r r r r r}
 &
 \multicolumn{1}{c}{Dim}&
 \multicolumn{1}{c}{Std}&
 \multicolumn{1}{c}{\textbf{\textsuperscript{\textdagger}OFMR}}&
 \multicolumn{1}{c}{\textbf{\textsuperscript{\textdagger}LAMR-100}}&
 \multicolumn{1}{c}{\textbf{FMR}}&
 \multicolumn{1}{c}{\textbf{1CMR}}&
 \multicolumn{1}{c}{\textbf{15MR}}&
 \multicolumn{1}{c}{\textbf{UCB/5}}&
 \multicolumn{1}{c}{\textbf{UCB/10}}&
 \multicolumn{1}{c}{\textbf{SAMR}}&
 \multicolumn{1}{c}{\textbf{GESMR}}&
 \multicolumn{1}{c}{\textbf{GESMR-AVG}}&
 \multicolumn{1}{c}{\textbf{GESMR-FIX}}
 \\
\toprule

\multirow{8}{*}{\rotatebox[origin=c]{90}{Ackley}}
 &\multirow{2}{*}{2}
 & 1
 & 0.1 & 0.1 & \textbf{0.0} & 0.1 & 0.0 & 0.0 & 0.0 & 0.0 & 0.0 & 0.0 & 0.0\\
&& 10
 & 1.2 & 1.2 & 2.2 & \textbf{0.2} & 0.3 & 0.3 & 0.4 & 0.3 & 0.5 & 0.3 & 0.4\\
 &\multirow{2}{*}{30}
 & 1
 & 1.9 & 1.3 & 3.1 & 2.9 & 1.7 & 2.6 & 3.1 & 2.7 & \textbf{1.7} & 3.6 & 3.5\\
&& 10
 & 8.7 & 5.9 & 15.7 & 15.4 & 7.4 & 8.2 & 7.8 & 12.4 & *\textbf{4.9} & 16.0 & 10.8\\
 &\multirow{2}{*}{100}
 & 1
 & 3.3 & 2.8 & 3.4 & 3.4 & \textbf{2.8} & 3.4 & 3.5 & 3.1 & 2.9 & 4.1 & 3.9\\
&& 10
 & 8.5 & 5.8 & 16.5 & 16.6 & 7.9 & 9.5 & 9.0 & 16.5 & *\textbf{6.8} & 17.4 & 14.3\\
 &\multirow{2}{*}{1000}
 & 1
 & 4.2 & 4.0 & 4.1 & 4.9 & 4.1 & 4.4 & 4.5 & 4.2 & \textbf{4.0} & 5.0 & 4.8\\
&& 10
 & 15.3 & 17.0 & 17.7 & 18.5 & 18.1 & 18.2 & 18.1 & 17.7 & *\textbf{17.6} & 18.4 & 18.3\\
\hline
\multirow{8}{*}{\rotatebox[origin=c]{90}{Griewank}}
 &\multirow{2}{*}{2}
 & 1
 & 0.0 & 0.0 & 0.0 & 0.0 & \textbf{0.0} & 0.0 & 0.0 & 0.0 & 0.0 & 0.0 & 0.0\\
&& 10
 & 0.0 & 0.0 & 0.0 & \textbf{0.0} & 0.0 & 0.0 & 0.0 & 0.0 & 0.0 & 0.0 & 0.0\\
 &\multirow{2}{*}{30}
 & 1
 & 0.1 & 0.1 & 0.3 & 0.1 & 0.1 & 0.2 & 0.2 & 0.1 & \textbf{0.1} & 0.4 & 0.3\\
&& 10
 & 0.7 & 0.5 & 1.4 & 1.3 & 0.8 & 0.7 & 0.8 & 0.6 & \textbf{0.5} & 1.2 & 1.0\\
 &\multirow{2}{*}{100}
 & 1
 & 0.2 & 0.1 & 0.2 & 0.2 & 0.1 & 0.2 & 0.2 & 0.1 & \textbf{0.1} & 0.5 & 0.4\\
&& 10
 & 1.3 & 0.6 & 2.5 & 2.5 & 0.7 & 0.9 & 1.0 & 0.8 & *\textbf{0.6} & 1.9 & 1.4\\
 &\multirow{2}{*}{1000}
 & 1
 & 0.5 & 0.5 & 0.6 & 1.1 & 0.6 & 0.8 & 0.8 & 0.6 & *\textbf{0.5} & 1.0 & 1.0\\
&& 10
 & 7.8 & 7.5 & 21.2 & 23.0 & 8.5 & 12.9 & 11.0 & 9.1 & \textbf{8.2} & 21.2 & 15.9\\
\hline
\multirow{8}{*}{\rotatebox[origin=c]{90}{Rastrigin}}
 &\multirow{2}{*}{2}
 & 1
 & 0.3 & 0.3 & 0.0 & 0.2 & \textbf{0.0} & 0.0 & 0.1 & 0.0 & 0.1 & 0.0 & 0.1\\
&& 10
 & 0.6 & 0.9 & 6.9 & \textbf{0.5} & 1.8 & 1.0 & 1.6 & 1.8 & 0.8 & 3.2 & 1.1\\
 &\multirow{2}{*}{30}
 & 1
 & 71.3 & 47.9 & 59.1 & \textbf{50.9} & 67.0 & 178.1 & 168.9 & 52.6 & 60.0 & 130.1 & 116.2\\
&& 10
 & 597.1 & 358.2 & 1606.6 & 1566.1 & 548.3 & 436.8 & 458.0 & 1204.7 & *\textbf{356.6} & 1575.0 & 677.7\\
 &\multirow{2}{*}{100}
 & 1
 & 301.5 & 182.1 & 218.9 & 215.0 & 227.6 & 796.5 & 772.1 & \textbf{181.3} & 198.4 & 543.1 & 483.8\\
&& 10
 & 2199.1 & 1500.1 & 7090.7 & 7080.6 & 2341.7 & 2303.2 & 1911.9 & 5164.7 & \textbf{1748.6} & 7433.1 & 3636.0\\
 &\multirow{2}{*}{1000}
 & 1
 & 4864.2 & 3918.3 & 4507.4 & 8541.2 & 4501.6 & 6713.5 & 9800.1 & 4520.6 & *\textbf{4187.1} & 8215.6 & 7901.7\\
&& 10
 & 4.5e+04 & 4.6e+04 & 9.1e+04 & 9.7e+04 & \textbf{5.7e+04} & 8.6e+04 & 6.4e+04 & 9.2e+04 & 6.7e+04 & 9.6e+04 & 9.4e+04\\
\hline
\multirow{8}{*}{\rotatebox[origin=c]{90}{Rosenbrock}}
 &\multirow{2}{*}{2}
 & 1
 & 0.0 & 0.0 & 0.0 & 0.0 & 0.0 & 0.0 & 0.0 & 0.0 & \textbf{0.0} & 0.1 & 0.0\\
&& 10
 & 0.9 & 1.1 & 5.8 & \textbf{0.5} & 1.1 & 1.3 & 2.0 & 1.7 & 0.7 & 2.0 & 1.0\\
 &\multirow{2}{*}{30}
 & 1
 & 485.7 & 262.3 & 977.9 & 386.4 & 323.8 & 1179.8 & 1643.0 & 339.9 & \textbf{243.9} & 1325.0 & 1064.5\\
&& 10
 & 2.1e+06 & 1.1e+06 & 1.5e+07 & 1.0e+07 & 1.1e+06 & \textbf{8.8e+05} & 1.2e+06 & 1.6e+06 & 9.1e+05 & 7.3e+06 & 2.2e+06\\
 &\multirow{2}{*}{100}
 & 1
 & 2685.3 & 1206.5 & 3769.2 & 3735.7 & 1390.2 & 1590.9 & 8349.1 & 1686.2 & \textbf{1360.1} & 7236.0 & 5483.7\\
&& 10
 & 1.9e+07 & 5.3e+06 & 9.3e+07 & 9.3e+07 & 6.4e+06 & 6.1e+06 & \textbf{5.6e+06} & 7.7e+06 & 5.8e+06 & 4.6e+07 & 1.7e+07\\
 &\multirow{2}{*}{1000}
 & 1
 & 6.2e+04 & 6.0e+04 & 8.9e+04 & 2.6e+05 & 7.3e+04 & 1.1e+05 & 8.8e+04 & 7.8e+04 & \textbf{6.6e+04} & 2.2e+05 & 2.0e+05\\
&& 10
 & 3.7e+08 & 3.7e+08 & 1.8e+09 & 2.2e+09 & 4.3e+08 & 1.2e+09 & 7.3e+08 & 4.7e+08 & \textbf{4.0e+08} & 1.5e+09 & 1.1e+09\\
\hline
\multirow{8}{*}{\rotatebox[origin=c]{90}{Sphere}}
 &\multirow{2}{*}{2}
 & 1
 & 0.0 & 0.0 & 0.0 & 0.0 & 0.0 & 0.0 & 0.0 & 0.0 & \textbf{0.0} & 0.0 & 0.0\\
&& 10
 & 0.1 & 0.1 & 0.1 & \textbf{0.0} & 0.0 & 0.0 & 0.0 & 0.0 & 0.0 & 0.0 & 0.0\\
 &\multirow{2}{*}{30}
 & 1
 & 3.2 & 1.8 & 6.3 & 2.2 & 1.8 & 3.1 & 5.9 & 1.8 & \textbf{1.5} & 8.6 & 5.9\\
&& 10
 & 318.7 & 180.9 & 1420.8 & 1153.4 & 182.3 & 162.9 & 163.5 & 212.7 & \textbf{142.9} & 1027.2 & 327.3\\
 &\multirow{2}{*}{100}
 & 1
 & 18.9 & 5.1 & 18.8 & 18.7 & 6.3 & 9.7 & 15.9 & 7.0 & \textbf{6.3} & 37.4 & 27.2\\
&& 10
 & 1892.3 & 512.6 & 6089.4 & 6088.1 & 668.8 & 1019.9 & 779.4 & 755.9 & \textbf{604.3} & 3600.2 & 1758.3\\
 &\multirow{2}{*}{1000}
 & 1
 & 270.4 & 259.4 & 382.4 & 809.3 & 304.2 & 314.0 & 319.9 & 327.1 & \textbf{286.3} & 772.1 & 711.8\\
&& 10
 & 2.7e+04 & 2.6e+04 & 8.1e+04 & 8.8e+04 & 3.0e+04 & 6.8e+04 & 4.6e+04 & 3.3e+04 & \textbf{2.9e+04} & 7.8e+04 & 6.0e+04\\
\hline
\multirow{8}{*}{\rotatebox[origin=c]{90}{Linear}}
 &\multirow{2}{*}{2}
 & 1
 & -836.1 & -835.5 & -45.4 & -434.1 & \textbf{-6.0e+27} & -766.9 & -710.5 & -1.2e+15 & -3.0e+17 & -1.2e+07 & -3.5e+05\\
&& 10
 & -1153.3 & -1150.8 & -385.6 & -764.9 & \textbf{-6.0e+27} & -1106.6 & -1050.8 & -7.8e+14 & -3.0e+17 & -1.2e+07 & -3.5e+05\\
 &\multirow{2}{*}{30}
 & 1
 & -3190.9 & -3238.2 & -171.9 & -240.2 & \textbf{-2.4e+28} & -2967.6 & -2671.8 & -5.5e+16 & -9.5e+17 & -1.1e+07 & -1.4e+06\\
&& 10
 & -4409.1 & -4413.7 & -1454.9 & -1523.2 & \textbf{-2.4e+28} & -4249.9 & -3954.8 & -3.9e+16 & -9.5e+17 & -1.2e+07 & -1.4e+06\\
 &\multirow{2}{*}{100}
 & 1
 & -5758.9 & -5665.5 & -277.3 & -275.9 & \textbf{-4.7e+28} & -5314.8 & -4956.6 & -2.9e+16 & -5.0e+18 & -1.4e+09 & -2.6e+06\\
&& 10
 & -7593.1 & -7473.6 & -2292.1 & -2290.6 & \textbf{-4.7e+28} & -7323.8 & -6971.4 & -2.2e+17 & -5.0e+18 & -1.4e+09 & -2.6e+06\\
 &\multirow{2}{*}{1000}
 & 1
 & -1.3e+04 & -1.3e+04 & -634.3 & -524.3 & \textbf{-9.3e+28} & -1.2e+04 & -1.1e+04 & -7.4e+16 & -6.3e+18 & -1.8e+10 & -5.7e+06\\
&& 10
 & -1.7e+04 & -1.7e+04 & -5244.2 & -5134.2 & \textbf{-9.3e+28} & -1.7e+04 & -1.6e+04 & -8.7e+16 & -6.3e+18 & -1.8e+10 & -5.7e+06\\

\bottomrule

\end{tabular}

}

\caption{
A genetic algorithm's \textit{average} (over generations) elite function value,
on various functions and population initializations 
using different mutation rate control strategies.
This metric quantifies how quickly the GA converged to good solutions.
The results are averaged over 40 seeds.
The best value is shown in bold.
A statistical t-test is performed on the best method and if the resulting p-value
is less than $0.05$ versus \textit{all} other methods, 
the result is considered significant and 
shown with an asterisk (*) in front of it. 
Methods marked with \textdagger~are oracles for benchmark
and are not compared against because they use foresight during evolution.
GESMR outperforms previous methods on most tasks, often significantly.
15MR outperforms GESMR in the Linear function landscapes
because 15MR directly doubles the MR every generation
while GESMR relies on a mutation that may double the MRs every generation.
}
\label{tbl::data_table_auc}
\end{table*}

\begin{table*}[t]
\centering

\scalebox{0.7}{

\begin{tabular}{c | r r | r r r r r r r r r r r}
 &
 \multicolumn{1}{c}{Dim}&
 \multicolumn{1}{c}{Std}&
 \multicolumn{1}{c}{\textbf{\textsuperscript{\textdagger}OFMR}}&
 \multicolumn{1}{c}{\textbf{\textsuperscript{\textdagger}LAMR-100}}&
 \multicolumn{1}{c}{\textbf{FMR}}&
 \multicolumn{1}{c}{\textbf{1CMR}}&
 \multicolumn{1}{c}{\textbf{15MR}}&
 \multicolumn{1}{c}{\textbf{UCB/5}}&
 \multicolumn{1}{c}{\textbf{UCB/10}}&
 \multicolumn{1}{c}{\textbf{SAMR}}&
 \multicolumn{1}{c}{\textbf{GESMR}}&
 \multicolumn{1}{c}{\textbf{GESMR-AVG}}&
 \multicolumn{1}{c}{\textbf{GESMR-FIX}}
 \\
\toprule

\multirow{8}{*}{\rotatebox[origin=c]{90}{Ackley}}
 &\multirow{2}{*}{2}
 & 1
 & 0.2 & 0.0 & *\textbf{5.3} & 38.6 & 29.0 & 31.5 & 44.2 & 44.8 & 15.5 & 65.1 & 47.3\\
&& 10
 & 0.7 & 0.0 & *\textbf{1.4} & 9.5 & 20.0 & 13.7 & 12.0 & 61.0 & 10.2 & 95.3 & 14.0\\
 &\multirow{2}{*}{30}
 & 1
 & 4.1 & 0.0 & 2.3 & 1.9 & 1.9 & 9.9 & 8.6 & 4.4 & \textbf{1.2} & 14.2 & 16.3\\
&& 10
 & 1.1 & 0.0 & 8.3 & 3.2 & 2.3 & 7.0 & 6.3 & 12.6 & \textbf{1.4} & 29.4 & 4.5\\
 &\multirow{2}{*}{100}
 & 1
 & 9.4 & 0.0 & 3.2 & 3.2 & 4.8 & 9.9 & 12.7 & 1.4 & \textbf{0.9} & 29.1 & 34.1\\
&& 10
 & 7.1 & 0.0 & 6.0 & 6.0 & 3.7 & 9.0 & 8.5 & 13.9 & *\textbf{0.4} & 77.1 & 14.7\\
 &\multirow{2}{*}{1000}
 & 1
 & 0.5 & 0.0 & 1.0 & 3.0 & 0.5 & 3.2 & 4.1 & 0.7 & *\textbf{0.2} & 105.4 & 28.7\\
&& 10
 & 11.7 & 0.0 & 3.1 & 9.0 & 14.5 & 16.5 & 11.7 & 5.1 & \textbf{3.0} & 89.9 & 23.1\\
\hline
\multirow{8}{*}{\rotatebox[origin=c]{90}{Griewank}}
 &\multirow{2}{*}{2}
 & 1
 & 0.0 & 0.0 & \textbf{5.3} & 38.6 & 58.5 & 17.3 & 16.4 & 16.0 & 20.3 & 15.5 & 47.6\\
&& 10
 & 0.5 & 0.0 & 7.9 & \textbf{2.8} & 8.1 & 9.5 & 7.9 & 76.5 & 5.1 & 76.0 & 5.0\\
 &\multirow{2}{*}{30}
 & 1
 & 1.2 & 0.0 & 1.3 & 1.7 & 3.1 & 7.9 & 7.7 & 0.8 & \textbf{0.6} & 8.4 & 18.6\\
&& 10
 & 2.2 & 0.0 & 6.9 & 2.8 & 4.1 & 8.6 & 7.7 & 0.9 & *\textbf{0.5} & 8.8 & 6.8\\
 &\multirow{2}{*}{100}
 & 1
 & 1.9 & 0.0 & 1.5 & 1.5 & 0.8 & 8.9 & 8.1 & 0.7 & *\textbf{0.2} & 9.0 & 23.1\\
&& 10
 & 2.6 & 0.0 & 5.6 & 5.6 & 0.9 & 7.8 & 8.0 & 0.7 & *\textbf{0.2} & 6.3 & 10.2\\
 &\multirow{2}{*}{1000}
 & 1
 & 0.3 & 0.0 & 0.4 & 7.4 & 0.4 & 4.2 & 3.7 & 1.0 & \textbf{0.2} & 37.1 & 18.0\\
&& 10
 & 0.4 & 0.0 & 8.0 & 26.0 & 0.5 & 7.7 & 7.3 & 0.7 & *\textbf{0.2} & 151.9 & 3.7\\
\hline
\multirow{8}{*}{\rotatebox[origin=c]{90}{Rastrigin}}
 &\multirow{2}{*}{2}
 & 1
 & 0.0 & 0.0 & *\textbf{5.3} & 38.6 & 39.7 & 30.1 & 44.2 & 21.9 & 12.6 & 25.4 & 47.5\\
&& 10
 & 0.8 & 0.0 & 9.7 & \textbf{0.9} & 11.7 & 3.2 & 3.1 & 80.6 & 6.0 & 99.4 & 2.6\\
 &\multirow{2}{*}{30}
 & 1
 & 2.1 & 0.0 & 2.4 & 6.0 & 2.8 & 22.1 & 24.4 & 1.4 & \textbf{1.2} & 7.0 & 31.9\\
&& 10
 & 8.6 & 0.0 & 6.4 & 4.7 & 16.9 & 16.0 & 16.2 & 6.1 & *\textbf{0.8} & 19.2 & 15.6\\
 &\multirow{2}{*}{100}
 & 1
 & 1.6 & 0.0 & 3.1 & 3.1 & 0.9 & 22.5 & 22.0 & 1.6 & \textbf{0.8} & 6.9 & 36.5\\
&& 10
 & 15.8 & 0.0 & 5.8 & 5.8 & 25.6 & 28.0 & 23.2 & 4.0 & *\textbf{0.9} & 27.4 & 27.9\\
 &\multirow{2}{*}{1000}
 & 1
 & 0.4 & 0.0 & 1.1 & 2.6 & 0.5 & 3.3 & 18.5 & 0.7 & \textbf{0.2} & 32.4 & 29.9\\
&& 10
 & 6.1 & 0.0 & 5.4 & 17.2 & 8.4 & 13.6 & 10.0 & 9.5 & *\textbf{2.9} & 118.7 & 13.5\\
\hline
\multirow{8}{*}{\rotatebox[origin=c]{90}{Rosenbrock}}
 &\multirow{2}{*}{2}
 & 1
 & 0.5 & 0.0 & \textbf{1.9} & 19.6 & 19.8 & 10.2 & 24.0 & 5.7 & 12.2 & 36.4 & 25.8\\
&& 10
 & 1.6 & 0.0 & 8.8 & \textbf{3.7} & 16.3 & 7.9 & 10.5 & 19.2 & 12.6 & 54.5 & 6.0\\
 &\multirow{2}{*}{30}
 & 1
 & 2.2 & 0.0 & 2.0 & 2.5 & 2.3 & 12.7 & 17.1 & 1.0 & \textbf{0.6} & 10.0 & 19.3\\
&& 10
 & 1.5 & 0.0 & 8.2 & 3.3 & 2.5 & 5.5 & 4.7 & 1.3 & \textbf{0.8} & 10.3 & 5.3\\
 &\multirow{2}{*}{100}
 & 1
 & 2.3 & 0.0 & 1.8 & 1.8 & 0.8 & 3.2 & 13.8 & 0.8 & \textbf{0.5} & 8.5 & 23.5\\
&& 10
 & 2.2 & 0.0 & 5.5 & 5.5 & 0.9 & 9.3 & 4.2 & 0.7 & *\textbf{0.3} & 10.3 & 9.4\\
 &\multirow{2}{*}{1000}
 & 1
 & 0.5 & 0.0 & 0.4 & 7.2 & 0.4 & 1.0 & 0.9 & 0.7 & \textbf{0.2} & 35.0 & 18.6\\
&& 10
 & 0.4 & 0.0 & 7.6 & 25.3 & 0.4 & 2.5 & 1.5 & 0.8 & *\textbf{0.2} & 15.0 & 3.8\\
\hline
\multirow{8}{*}{\rotatebox[origin=c]{90}{Sphere}}
 &\multirow{2}{*}{2}
 & 1
 & 0.5 & 0.0 & *\textbf{4.1} & 35.0 & 35.9 & 31.1 & 17.8 & 65.9 & 11.0 & 205.5 & 43.3\\
&& 10
 & 0.4 & 0.0 & *\textbf{0.4} & 12.1 & 19.0 & 14.0 & 7.9 & 53.7 & 9.7 & 106.2 & 17.2\\
 &\multirow{2}{*}{30}
 & 1
 & 3.0 & 0.0 & 2.2 & 4.0 & 2.6 & 7.0 & 15.7 & \textbf{0.6} & 0.7 & 8.8 & 24.6\\
&& 10
 & 2.7 & 0.0 & 6.8 & 3.0 & 2.1 & 8.0 & 6.9 & 1.0 & \textbf{0.7} & 10.2 & 7.7\\
 &\multirow{2}{*}{100}
 & 1
 & 3.2 & 0.0 & 3.2 & 3.2 & 1.2 & 2.5 & 4.7 & 0.6 & *\textbf{0.5} & 7.2 & 30.8\\
&& 10
 & 4.8 & 0.0 & 5.7 & 5.7 & 1.5 & 14.6 & 5.5 & 0.6 & \textbf{0.5} & 7.5 & 14.4\\
 &\multirow{2}{*}{1000}
 & 1
 & 0.4 & 0.0 & 0.5 & 7.9 & 0.4 & 1.0 & 1.0 & 0.7 & \textbf{0.2} & 54.6 & 17.4\\
&& 10
 & 0.4 & 0.0 & 8.1 & 26.2 & 0.4 & 3.5 & 1.8 & 0.7 & \textbf{0.2} & 61.1 & 3.6\\
\hline
 
\bottomrule

\end{tabular}

}

\caption{
A genetic algorithm's mean squared error log MR compared to empirical estimate of
the long-term optimal log MR (the log MR from LAMR-100),
on various functions and population initializations 
using different mutation rate control strategies.
This metric quantifies how optimal (lower is better) the MRs produced are for the long-term.
The results are averaged over 40 seeds.
The best value is shown in bold.
A statistical t-test is performed on the best method and if the resulting p-value
is less than $0.05$ versus \textit{all} other methods, 
the result is considered significant and 
shown with an asterisk (*) in front of it. 
Methods marked with \textdagger~are oracles for benchmark
and are not compared against because they use foresight during evolution.
GESMR consistently outperforms other methods,
showing that GESMR is producing MRs optimal for the long-term.
The Linear function is not shown because LAMR-100 is not
able to produce the true optimal MR (goes to infinity),
so comparisons to LAMR-100 in a Linear function does not make sense.
}
\label{tbl::data_table_mselog}
\end{table*}

\section{Details of the Function Optimization Experiment}

Detailed definitions of the test functions are given in this Appendix, followed by detailed results.

\subsection{Test Function Definitions}

\begin{description}
\vspace*{1ex}
    \item \textbf{Ackley}:
\vspace*{-1ex}
\begin{align}
    f(x) = 
    &-a\exp{\akparen{-b\sqrt{\frac1d \sum_{i=1}^d x_i^2}}}\\
    &-\exp{\akparen{\frac1d \sum_{i=1}^d \cos{(cx_i)}}}+a+\exp{(1)},
\end{align}
with $a=20, b=0.2, c=2\pi$.

\vspace*{1ex}
\item \textbf{Griewank}:
\vspace*{-1ex}
\begin{align}
    f(x) = &\sum_{i=1}^d\frac{x_i^2}{4000}-\prod_{i=1}^d \cos{\akparen{\frac{x_i}{\sqrt{i}}}}+1.
\end{align}

\vspace*{1ex}
\item \textbf{Rastrigin}:
\vspace*{-1ex}
\begin{align}
    f(x) = &10d+\sum_{i=1}^d\akbrack{x_i^2-10\cos{(2\pi x_i)}}.
\end{align}

\vspace*{1ex}
\item \textbf{Rosenbrock}:
\vspace*{-1ex}
\begin{align}
    f(x) = &\sum_{i=1}^{d-1} \akbrack{100(x_{i+1}-x_i^2)^2+(x_i-1)^2}.
\end{align}

\vspace*{1ex}
\item \textbf{Sphere}:
\vspace*{-1ex}
\begin{align}
    f(x) = &\sum_{i=1}^d x_i^2.
\end{align}

\vspace*{1ex}
\item \textbf{Linear}:
\vspace*{-1ex}
\begin{align}
    f(x) = &\sum_{i=1}^d x_i.
\end{align}

\end{description}

\subsection{Function Optimization Results}
The full results of the test optimization functions
are shown in 
Tables~\ref{tbl::data_table},~\ref{tbl::data_table_auc},~\ref{tbl::data_table_mselog}.
Table~\ref{tbl::data_table} summarizes
the final elite function value achieved by each algorithm in all the 
test function optimization runs.
Table~\ref{tbl::data_table_auc} summarizes
the average elite function value over generations 
from each algorithm in all the 
test function optimization runs.
Table~\ref{tbl::data_table_mselog} summarizes
the mean squared error between the average log MR
of a given algorithm with the log MR of
LAMR-$100$ (the oracle long-term MR).
These result show that
GESMR outperforms other methods 
in the high dimensional and rugged function
landscapes.
GESMR also produces MRs that match the oracle
long-term optimal MR,
showing that GESMR empirically produces MRs
suited for the long-term.
GESMR also scales well to the high dimensions
of neuroevolution.

\clearpage
\clearpage
\clearpage
\clearpage

\section{Details of the Image Classification Experiment}
MNIST and Fashion-MNIST are common image classification
\linebreak
datasets of hand written digits and clothes, respectively 
\cite{LeCun1998-wt, Xiao2017-kj}.
The inputs are 28$\times$x28 grayscale images and the output is one of ten classification labels.
Both datasets consists of 60,000 training images 10,000 evaluation images.
For these problems, $f$ is the 
negative log-likelihood function (i.e.\ the cross-entropy loss) 
as is common in supervised learning.

The evolved neural-network architecture contains three 3$\times$3
\linebreak
Conv2D layers with 10 channels, each one followed
by a 2$\times$2 MaxPooling layer and a ReLU nonlinearity.
The resulting feature maps are collapsed into a vector
and fed into a 10$\times$10 Dense layer followed by a 
ReLU and another 10$\times$10 Dense layer. after which they are fed into a Softmax function to output ten class probabilities.

\section{Details of the Reinforcement Learning Experiment}
CartPole, Pendulum, Acrobot, and MountainCar
are common reinforcement learning control tasks.
In each of these tasks, the performance of a robot controller is evaluated
in a simulated environment \cite{Brockman2016-bx}.
CartPole consists of balancing a single pole on a one-dimensional cart
for as long as possible or until 200 timesteps have passed,
rewarded for how long the pole stays up.
Pendulum consists of a robot trying to swing up a pendulum,
rewarded for maintaining as much of an upward angle
as possible.
Acrobot consists of moving a joint with two links
such that the bottom link swings to as high as possible.
MountainCar consists of a car with a weak engine in valley between two hills;
it must be moved back and forth between the hills to gain enough energy to reach the top of the target hill.
In all environments, $f$ is the 
negative cumulative reward of an episode 
(averaged over five episodes).

The evolved neural-network architecture contains a dense layer to 
map the number of observations to 128 hidden neurons with a
ReLU activation function, and another dense layer mapping the 128 neurons to the number of actions.
If the action space is discrete, a Softmax function 
is applied to output action probabilities.

The detailed results are shown in Figure~\ref{fig:fits_mrs_rl}.

\begin{figure}[t]
\centering
\includegraphics[width=\columnwidth]{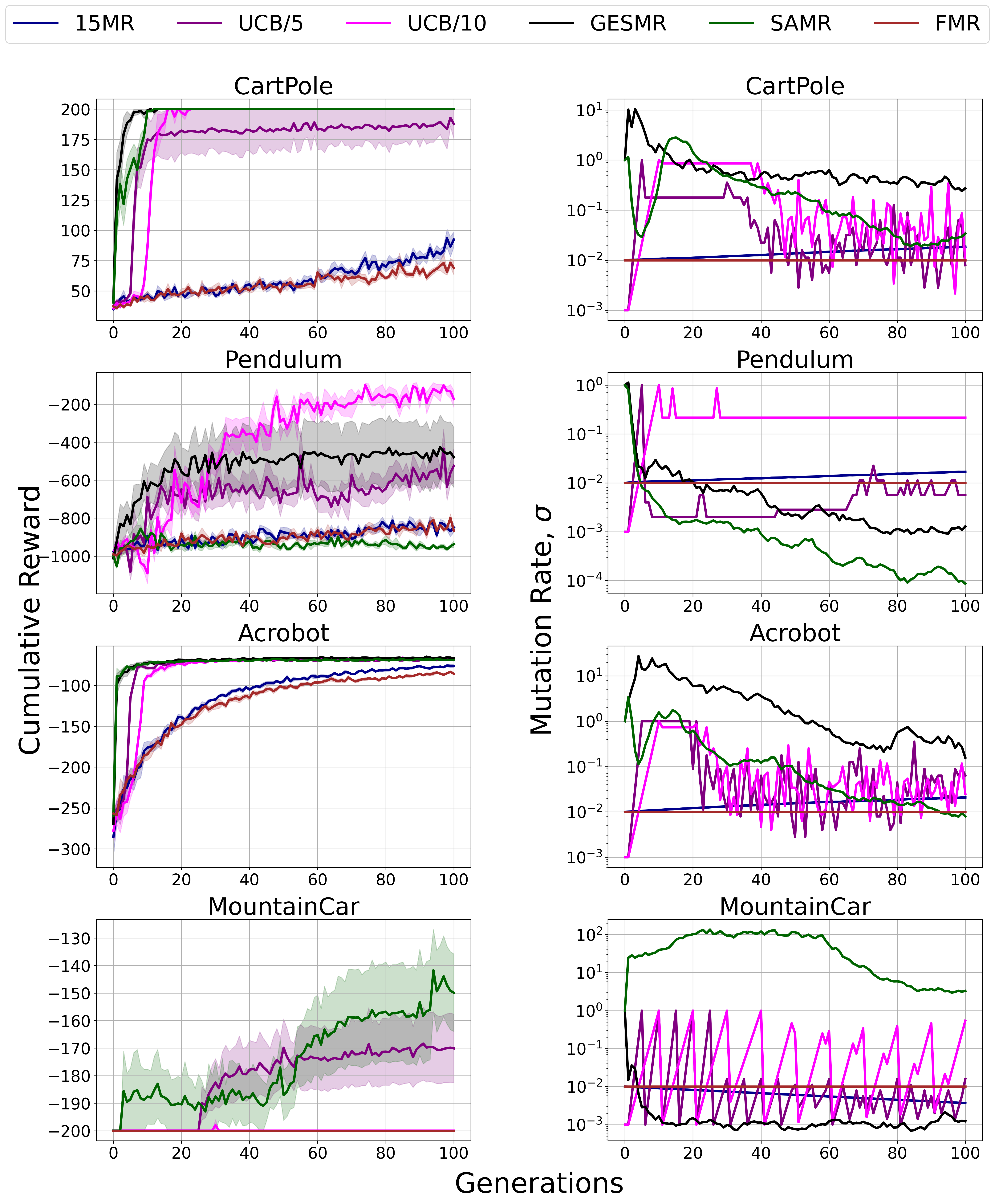}
\caption{
Elite function value and average mutation rate (MR)
over generations of neuroevolution
using different MR control strategies
applied to the reinforcement learning control tasks
of CartPole, Pendulum, Acrobot, and MountainCar.
GESMR outperforms most other methods in CartPole,
Pendulum, and Acrobot, but fails in MountainCar.
}
\label{fig:fits_mrs_rl}
\end{figure}

\end{document}